\RequirePackage{fix-cm}
\documentclass[twocolumn]{svjour3}          
\smartqed  
\usepackage{graphicx}
\usepackage{subcaption}
\usepackage{algorithmic}
\usepackage{bm}
\usepackage{ulem}
\usepackage[ruled,vlined,linesnumbered,ruled]{algorithm2e}

\usepackage{amssymb}
\usepackage{amsmath}

\usepackage{xcolor}
\newcommand{\zy}[1]{{#1}}
\newcommand{\zyw}[1]{{#1}}

\usepackage{booktabs}
\usepackage{multirow}
\usepackage{makecell}

\begin{document}

\title{FedST: Secure Federated Shapelet Transformation for Time Series Classification
}


\author{Zhiyu Liang         \and
        Hongzhi Wang 
}


\institute{Zhiyu Liang.  \at
              Harbin Institute of Technology, Harbin, China \\
              \email{zyliang@hit.edu.cn}           
           \and
           Hongzhi Wang. \at
           Harbin Institute of Technology, Harbin, China \\
           \email{wangzh@hit.edu.cn}   
}

\date{Received: date / Accepted: date}

\maketitle

\begin{abstract}
\zy{This paper explores how to build a shapelet-based time series classification (TSC) model in the federated learning (FL) scenario, that is, using more data from multiple owners without actually sharing the data. We propose FedST, a novel federated TSC framework extended from a centralized shapelet transformation method. We recognize the federated shapelet search step as the kernel of FedST. Thus, we design a basic protocol for the FedST kernel that we prove to be secure and accurate. However, we identify that the basic protocol suffers from efficiency bottlenecks and the centralized acceleration techniques lose their efficacy due to the security issues. To speed up the federated protocol with security guarantee, we propose several optimizations tailored for the FL setting. Our theoretical analysis shows that the proposed methods are secure and more efficient. We conduct extensive experiments using both synthetic and real-world datasets. Empirical results show that our FedST solution is effective in terms of TSC accuracy, and the proposed optimizations can achieve three orders of magnitude of speedup.}
\keywords{Time series classification \and Federated Learning \and Time series features \and Time series shapelets}
\end{abstract}

\section{Introduction}

\textit{Time series classification} (TSC) aims to predict the class label for given time series samples. It is one of the most important \zyw{problems} for data analytics, with applications in various scenarios~\cite{susto2018time,dheepadharshani2019multivariate,ramirez2019computational}. 

Despite the impressive performance existing TSC algorithms have been achieving~\cite{bagnall16bakeoff,ismail2019deep,abanda2019review,ruiz2021great,middlehurst2021hive,tang2021omni,dempster2021minirocket,tan2022multirocket}, they usually make an ideal assumption that the user has free access to enough labeled data. However, it is quite difficult to collect and label the time series for real-world applications.  

For instance, small manufacturing businesses monitor their production lines using sensors to analyze the working \zyw{conditions}. Since the data sequences related to specific conditions, e.g., a potential failure of an instrument, are usually rare pieces located in unknown regions of the whole monitoring time series, the users have to manually identify the related pieces for labeling, which can be expensive due to the need \zyw{for} professional knowledge. As a consequence, it is costly for these businesses to benefit from the advanced TSC solutions, as they \zyw{do not have} enough labeled data to learn accurate models. 

To deal with the problem, a natural idea is to enrich the local training data by gathering the labeled samples from external data sources, e.g., the other businesses that run the same instrument. However, it has been increasingly difficult for organizations to combine their data due to privacy concerns~\cite{yang2019federated,voigt2017eu}.

\subsection{Motivation}\label{sec:motivation}

\zy{To solve the above problem, a new paradigm named \textit{Federated Learning} (FL)~\cite{mcmahan2017communication} has recently been proposed. FL aims to allow multiple businesses to jointly train a model without revealing their private data to each other. An example of using FL to enrich the training time series is shown in Fig.~\ref{fig:fedtsc}. However, existing FL solutions focus on the training of general models, including tree models~\zy{\cite{wu13privacy,fu2021vf2boost,fang2021large,cheng2021secureboost}}, linear models~\zy{\cite{nikolaenko2013privacy,mohassel2017secureml,aono2016scalable}}, and neural networks~\zy{\cite{mcmahan2017communication,shokri2015privacy,mcmahan2017learning,fu2022blindfl}}, which have limitations for the TSC problem. The main reasons are as follows.}

\begin{figure}
	\centering
	\includegraphics[width=.75\linewidth]{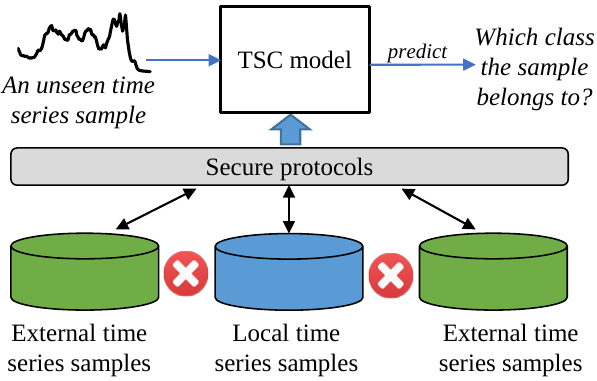}
	
	\caption{Example of enabling federated learning to enrich the training time series data. A business who owns some training time series samples (blue) collaborates with the partners who have additional training samples (green) to jointly build the TSC model. They follow some secure protocols to avoid disclosing their private training data.}
	\label{fig:fedtsc}
\end{figure}

First, the tree-based and linear classifiers are shown \zyw{to be} weak in capturing the temporal patterns for classifying time series~\cite{bagnall16bakeoff}, while the accuracy of neural networks usually relies on the hyper-parameter tunning, which is still a challenging problem in the FL scenario. Second,  many real-world TSC applications~\cite{ghalwash2013extraction,ye2011time,ramirez2019computational,perez2015fast} expect the classification decisions to be explainable/interpretable, e.g., the users know why a working condition is determined as a fault. However, a time series usually has a large number of data points (e.g., 537 on average for the 117 fixed-length datasets of the UCR Archive \cite{DBLP:journals/corr/abs-1810-07758}), which are taken as independent variables by the general models. \zyw{It will be difficult to explain the classification decisions with so many input variables.}   

\zyw{Faced with} the above limitations, we propose to customize FL solutions for the TSC problem by extending the centralized TSC approaches to the federated setting. To achieve this goal, we have proposed FedTSC~\cite{fedtsc}, a brand new FL system tailored for TSC, and \zyw{have} demonstrated its utility in VLDB. In this paper, we elaborate on the design ideas and essential techniques of a main internal of the system, i.e., the novel \underline{Fed}erated \underline{S}hapelet \underline{T}ransformation (FedST) framework. We design FedST based on the centralized shapelet transformation method due to the following benefits.

\begin{figure}
	\centering
	\includegraphics[width=.7\linewidth]{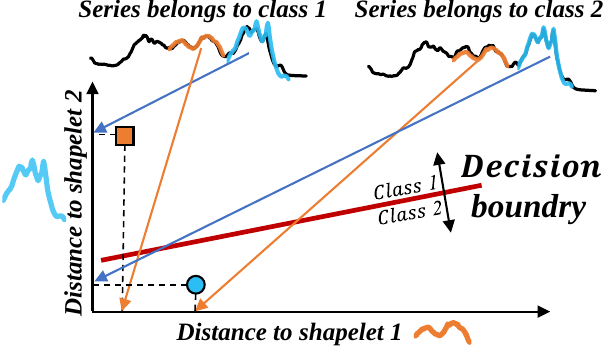}
	
	\caption{Illustration of the shapelet-based features. A shapelet is a salient subsequence that represents \zyw{a} shape unique to certain classes. With a few shapelets of high distinguishing ability, each time series sample is transformed into a low-dimensional feature vector representing how similar (distant) the sample is to these shapelets. The classification is made and explained based on the few features rather than the abundant \zyw{data points} of the raw time series. In this example, the time series similar to shapelet 1 (orange) and distant to shapelet 2 (blue) are classified into class 1 and vice versa.}
	\label{fig:interpretability}
 \vspace{-1ex}
\end{figure}

\noindent
\zy{\textbf{Our design choice.}} First, the shapelet transformation method not only achieves competitive accuracy over existing centralized TSC approaches~\cite{bagnall16bakeoff}, but also serves as an essential component of the ensemble classifier named HIVE-COTE 2.0 (HC2), which is currently state-of-the-art centralized TSC model~\cite{middlehurst2021hive}. Second, the method adopts the shapelet-based features rather than the raw time series as input to the classification models. The features represent the similarity between the time series and a set of shapelets (i.e., the salient subsequences), which can be order-of-magnitude less in number compared to the raw data points and are very intuitive to understand~\cite{hills2014classification}. Thus, building classifiers on top of the shapelet-based features can simplify the explanation. Third, the shapelets used to transform the raw time series can be extracted in an anytime manner to flexibly balance the accuracy and the efficiency (see Sec.~\ref{fedst_kernel}), which are beneficial for practical utility. Fig.~\ref{fig:interpretability} is an illustration of the shapelet-based features.

\zy{One worry of our design choice may be the scalability issue. The original shapelet transformation method has a quadratic time complexity with respect to the number of training instances~\cite{hills2014classification}, by selecting shapelets from all possible subsequences of the training time series, and thus cannot scale well to more training data.  However, existing studies have shown that it is never necessary to enumerate all possible subsequences~\cite{bagnall2020tale}. Instead, the running time can be limited within a moderate constant (e.g., 10 hours for every UCR dataset~\cite{middlehurst2021hive,bagnall2020tale}) to achieve considerable accuracy, benefiting from the anytime property of the algorithm (see Sec.~\ref{fedst_kernel}). Therefore, in this paper, we dedicate ourselves to extending the shapelet transformation method to the FL scenario, and we will also consider other advanced TSC methods that can be more scalable and complementary to our solution in our future work to benefit from both.}

\subsection{Challenges and contributions}

Although it is practical to extend the centralized approach to the federated setting, it is unexplored how to achieve both security and efficiency during the federated shapelet search (FedSS) step, which is the kernel of the FedST framework (see Sec.~\ref{fedst_kernel} in detail). 

The goal of the federated shapelet search is to jointly utilize the distributed labeled time series to find the shapelets with the highest quality for distinguishing the classes. To ensure \zyw{the} security of the federated computation, a natural idea is to extend the centralized shapelet search using secure multi-party computation (MPC)~\cite{yao1982protocols,damgaard2012multiparty,keller2020mp,keller2013architecture}. Following that, we first develop $\Pi_{FedSS-B}$, the basic protocol to achieve FedSS. Benefiting from MPC, we show that this protocol is secure and effective. 

However, by our analysis, the basic protocol suffers from low efficiency due to the high overhead incurred by MPC during the \textit{shapelet distance computation} and the \textit{shapelet quality measurement} stages. Although there are acceleration techniques in the centralized scenario~\cite{mueen2011logical,keogh2006lb_keogh,ye2011time,rakthanmanon2012searching}, we prove that these methods are insecure in the FL setting and \zyw{therefore are} unfeasible. Consequently, we propose acceleration methods tailored for the FL setting with security guarantee to tackle the efficiency bottlenecks of $\Pi_{FedSS-B}$. 

For shapelet distance computation, we identify the Euclidean norm computation as the efficiency bottleneck, so we propose a speed-up method based on a novel secure dot-product protocol. For quality measurement, we first design an optimization to reduce the duplicated time-consuming interactive operations with secure sorting. Then, we propose to further boost the efficiency through an acceptable trade-off of classification accuracy. We show both \textit{theoretically} and \textit{empirically} the effectiveness of these techniques.


\vspace{1ex}
{\setlength{\parindent}{0cm}
	\textbf{Contributions.} We summarize our contributions as follows.
	
	\begin{enumerate}
		\item We investigate the customized FL solution for time series classification. In particular, we propose FedST, the first shapelet-based FL method which extends the centralized shapelet transformation to the federated scenario to make use of its advantages in terms of accuracy, interpretability, and flexibility.
		
		\item We present $\Pi_{FedSS-B}$, a basic federated protocol for the FedST kernel, i.e., the federated shapelet search, which adopts MPC to achieve security. We analyze the protocol in terms of security, effectiveness, and efficiency. We identify the efficiency bottlenecks of $\Pi_{FedSS-B}$ and the invalidity of the centralized speed-up techniques due to the security issue. To boost the protocol efficiency, we propose acceleration methods tailored for the FL setting, which are theoretically secure and are more scalable and efficient than the basic protocol.
		
		\item We conduct extensive experiments to evaluate our solutions, which \zyw{have} three major observations. (1) Our FedST offers superior accuracy comparable to the non-private approach. (2) Each of our proposed acceleration approaches is individually effective, and they together bring up to three orders of magnitude of speedup. (3) The proposed trade-off method provides up to 8.31x speedup over our well-optimized protocol while guaranteeing comparable accuracy. We further demonstrate the interpretabiltiy and flexibiltiy of our framework.
	\end{enumerate}
}

{\setlength{\parindent}{0cm}
	\textbf{Organization.} We introduce the preliminaries in Sec.~\ref{sec:pre}. We propose the FedST framework and talk about the FedST kernel, i.e., federated shapelet search, in Sec.~\ref{solution_overview}. The basic protocol of the federated shapelet search is presented and analyzed in Sec.~\ref{fedss_bs}. We elaborate on the acceleration methods tailored for the two efficiency bottlenecks of the basic protocol in Sec.~\ref{Distance_Acceleration} and~\ref{measurement_acceleration}, respectively. We show \zyw{the} experimental results in Sec.~\ref{exp}. We illustrate how to incorporate differential privacy to further enhance the security in Sec.~\ref{dp-protect} and conclude this paper in Sec.~\ref{sec:conclusion}. 
}
	\section{Related Work}
Our work is related to federated learning, feature-based time series classification, and privacy protection.

	\subsection{Federated Learning}\label{sec:survey-FL}
	Recently, there have been numerous works \zyw{dedicated} to the federated learning of the general models, including the linear models~\zy{\cite{nikolaenko2013privacy,mohassel2017secureml,aono2016scalable}},  the tree models~\zy{\cite{wu13privacy,fu2021vf2boost,fang2021large,cheng2021secureboost}}, and the neural networks~\zy{\cite{mcmahan2017communication,shokri2015privacy,mcmahan2017learning,fu2022blindfl}}. However, none of them \zyw{achieves} the same goal as our solution, because these general models have limitations in tackling the TSC problem~\cite{bagnall16bakeoff} in terms of accuracy and interpretability. There are also FL solutions designed for specific tasks~\cite{mcmahan2017learning,liu2020secure,wang2021efficient,huang2021personalized,10.14778/3494124.3494125,li2021federated,tong2022hu,muhammad2020fedfast,li2021privacy,chen2022fedmsplit}. These methods target scenarios that are completely different from ours. As a result, we propose to tailor \zyw{the} FL method for TSC. In specific, we contribute to proposing the secure FedST framework to take advantage of the shapelet transformation in terms of accuracy, interpretability, and flexibility, and addressing the security and efficiency issues within the framework. 

 \zy{Note that the generic FL frameworks, such as the popular FedAvg~\cite{mcmahan2017communication} and its customized variants~\cite{younis2023flames2graph,xing2022efficient-multitask-fedTSC}, which can train any stochastic gradient descent (SGD) based model (e.g. deep neural networks~\cite{lecun2015deep-DL-Book}) across the data federation, can also solve the federated TSC problem by training centralized TSC models such as spatial-temporal convolutional neural network~\cite{st-CNN} and ResNet~\cite{ResNet-TSC} in the FL setting. This kind of FL solution is a standard yet very strong baseline in terms of TSC accuracy. However, the generic framework relies on a secure broker to aggregate the models or gradients of the parties, which is costly and can disclose sensitive data in practice~\cite{tong2022hu}. In comparison, we show in Sec.~\ref{exp} that our customized solution can achieve competitive accuracy without using such a broker, and has nice properties in terms of interpretability and flexibility, which are beneficial for practical utility.}
	

\subsection{Feature-based Time Series Classification} 

Instead of directly building classifiers upon the raw time series, transforming the time series into low-dimensional or sparse feature vectors can not only achieve competitive classification accuracy, but also simplify the explanation. 

In summary, there are three types of TSC methods based on different explainable features, i.e., the shapelet-based methods~\cite{ye2011time,mueen2011logical,hills2014classification,bostrom2017binary,grabocka2014learning,li2021shapenet,liang2021efficient} that determine the class labels based on the localized shapes, the interval-based methods~\cite{middlehurst2020canonical,cabello2020fast,middlehurst2021hive} that classify the time series based on the statistics \zyw{in} some specific time ranges, and the dictionary-based approaches~\cite{le2017time,large2019time,middlehurst2019scalable,tde} that utilize the pattern frequency as features. These types of methods can \zyw{complement} each other to contribute to \zyw{the} state-of-the-art accuracy~\cite{lines2018time,bagnall2020tale,middlehurst2021hive}. This work focuses on developing a novel framework with a series of optimization techniques taking advantage of the shapelet-based approaches, while we would like to present our contributions~\cite{fedtsc} of enabling FL for interval-based and dictionary-based TSC in the future.
 

Shapelet-based TSC is first proposed by~\cite{ye2011time}. In the early \zyw{work}, shapelets are discovered in company with a decision tree training, where a shapelet is found at each tree node to determine the best split of the node~\cite{ye2011time,mueen2011logical,lines2012alternative}. To benefit from the other classifiers, a shapelet transformation framework~\cite{hills2014classification} is proposed that decouples the shapelet discovery from the decision tree training and produces a transformed dataset that can be used in conjunction with any classifier. Several works are raised to speedup the shapelet search~\cite{ye2011time,keogh2006lb_keogh,mueen2011logical,rakthanmanon2012searching} and improve the shapelet quality~\cite{bostrom2017binary}. 

Another line of \zyw{work dedicates} to jointly learning the shapelets and the classifiers~\cite{grabocka2014learning,liang2021efficient,ma2019triple,li2021shapenet,ma2020adversarial,fang2018efficient,hou2016efficient}. However, the learning-based methods are much more complex because they incur several additional hyperparameters that highly affect the accuracy. Besides, they are inflexible due to the coupling of the shapelet and classifier, and cannot run in the anytime fashion to trade off the classification accuracy and the efficiency.

Based on the above discussions, we take advantage of the shapelet transformation method~\cite{hills2014classification,bostrom2017binary,bagnall2020tale} to develop our FL solution. \zyw{However}, our work differs from existing studies because we carefully consider the security and efficiency issues in a brand new FL scenario.

\subsection{Privacy Protection} Data privacy is one of the most essential problems in FL~\cite{yang2019federated,li2020federated,kairouz2021advances}. Several techniques have been studied \zyw{in existing work}. Secure Multi-Party Computation~\cite{yao1982protocols} is a general framework that offers secure protocols for many arithmetic operations~\cite{damgaard2012multiparty,keller2020mp,keller2013architecture}. These operations are efficient for practical utility~\cite{aly2019benchmarking,chen2019secure,li2020practical,mohassel2017secureml,li2021privacy,wu13privacy} under the semi-honest model that most FL works consider, while they can also be extended to the malicious model through zero-knowledge proofs~\cite{goldreich1994definitions}. 

Homomorphic Encryption (HE) is another popular technique in FL~\cite{cheng2021secureboost,fu2021vf2boost,10.14778/3494124.3494125,zhang2020batchcrypt,wu13privacy}, which allows \zyw{for} a simple implementation of the secure addition. However, HE does not support some complex operations (e.g., division and comparison). The encryption and decryption are also computationally intensive~\cite{wu13privacy,fu2021vf2boost}.

Compared to the solutions based on MPC and HE \zyw{that} aim to protect the intermediate information during the federated computation,  an orthogonal line of \zyw{work adopts} the Differential Privacy (DP) to protect the privacy for the outputs, such as the parameters of the learned models. It works by adding \zyw{noise} to the private data~\cite{wang2021efficient,wei2020federated,liu2021projected,li2021federated,Pan2022FedWalkCE} to achieve a trade-off between the precision and the degree of privacy for a target function. Thus, DP can usually complement MPC and HE. 

In this paper, we mainly adopt MPC to ensure no intermediate information is disclosed during the complex computations of FedST, because it provides the protocols for the required arithmetic operations. We also illustrate that the private data can be further protected with privacy guarantee by incorporating DP.

\section{Preliminaries}\label{sec:pre}

This section presents the preliminaries, including the target problem of the paper and the two building blocks of the proposed FedST, i.e., the shapelet transformation and the secure multi-party computation. \zy{We begin by summarizing the main notations in Table~\ref{tab:notions}. } 

\subsection{Problem Statement}

Time series classification (TSC) is the problem of creating a function that maps from the space of input time series samples to the space of class labels~\cite{bagnall16bakeoff}. A time series (sample) is defined as a sequence of data points
$ T = (t_1,\ldots,t_p,\ldots,t_N )$
ordered by time, where $t_p$ is the observation at timestamp $p$, and $N$ is the length. The class label $y$ is a discrete variable with $C$ possible values. i.e., $y \in \{c\}_{c=1}^C$ where $C \ge 2$. 

Typically, TSC is achieved by using a training data set $TD=\{(T_j, y_j)\}_{j=1}^M$ to build a model that can output either predicted class values or class distributions for previously unseen time series samples, where the instance $(T_j, y_j)$ represents the pair of the $j$-th time series sample and the corresponding label.

Specifically, in this paper we target the TSC problem in a federated setting, denoted as the FL-enabled TSC problem defined as follows.

\begin{definition}[FL-enabled TSC problem]
Given a party $P_0$ (named initiator) who owns a training data set $TD^0$ and $n - 1$ partners $P_1, \ldots, P_{n-1}$ (named participants) who hold the labeled series $TD^1, \ldots, TD^{n-1}$ collected from the same area (e.g., monitoring the same type of instruments), where $TD^i = \{(T^i_j, y^i_j)\}_{j=1}^{M_i}$, the goal of the problem is to coordinate the parties to build TSC models $\mathcal{M}$ for the initiator $P_0$ without revealing the local data $TD^0$, $\ldots$, $TD^{n-1}$ to each other.   
\end{definition}

Note that every party in the group can act as the initiator to benefit from the federated learning. For ease of exposition, we denote $\sum_{i=0}^{n-1}M_i = M$ and $\bigcup_{i=0}^{n-1}TD^i=TD$. Ideally, the performance of $\mathcal{M}$ should be lossless compared to that of the model trained in a centralized scenario using the combined data $TD$.

Similar to previous FL \zyw{work}~\cite{wu13privacy,fu2021vf2boost,fu2022blindfl,li2021federated,tong2022hu}, we consider the semi-honest model where each party follows the protocols but may infer the private information from the received messages, while our method can be extended to the malicious model through zero-knowledge proofs~\cite{goldreich1994definitions}. Unlike existing studies that usually conditionally allow \zyw{the} disclosure of some private data~\cite{fu2021vf2boost,10.14778/3494124.3494125}, we adopt a stricter security definition~\cite{wu13privacy,mohassel2017secureml} to ensure \textit{no intermediate information is disclosed}.   

\begin{definition}[Security]
	Let $\mathcal{F}$ be an \textit{ideal} functionality such that the parties send their data to a trusted party for computation and receive the final results from the party. Let $\Pi$ be a \textit{real-world} protocol executed by the parties. We say that $\Pi$ securely realizes $\mathcal{F}$ if for each adversary $\mathcal{A}$ attacking the real interaction, there exists a simulator $\mathcal{S}$ attacking the ideal interaction, such that for all environments $\mathcal{Z}$, the quantity $\arrowvert \Pr[REAL(\mathcal{Z}, \mathcal{A}, \Pi, \lambda) = 1] - \Pr[IDEAL(\mathcal{Z}, \mathcal{S}, \mathcal{F}, \lambda) = 1] \arrowvert$ is negligible (in $\lambda$).
	\label{definition:security}
\end{definition}

\begin{table}[t]
    \centering
      \caption{\zy{Summary of notations used in this paper.}}
    \begin{tabular}{|l|l|}
    \toprule
       \textbf{Sign}  & \textbf{Specification} \\
         \midrule
      $T$   &  A time series sample of length $N$ \\
      \hline
      $P_i$ & The $i$-th party \\
      \hline
      $TD^i$ ($TD$) & The time series dataset of $P_i$ (all parties) \\
      \hline
      $(T^i_j,y^i_j)$ & \makecell[l]{The $j$-th instance in $TD^i$ where $T^i_j$ is the \\ time series and $y^i_j$ corresponds to the label} \\
      \hline
      $M_i$ ($M$) & The number of instances in $TD^i$ ($TD$) \\
      \hline
      $\{c\}_1^{C}$ & The label set of size $C$ \\
      \hline
      $\mathcal{F}$/$\mathcal{A}$ & An ideal functionality/adversary\\
      \hline
      $\Pi$/$\mathcal{S}$ & A real protocol/simulator\\
      \hline
      $t_{j,p}$ & The value of $T_j$ in the $p$-th timestamp\\
      \hline
      $S$ & The shapelet of length $L_S < N$\\
      \hline
      $T_j[s,l]$ & \makecell[l]{The subsequence of $T_j$ starting at the\\timestamp $s$ and lasting the length $l$}\\
      \hline
      $d_{T_j,S}$ & \makecell[l]{The distance between the time series $T_j$\\and the shapelet $S$}\\
      \hline
      \makecell[l]{$D_S$\\ ($D_{S,c}$) }& \makecell[l]{The set of distances between the shapelet $S$\\and the time series (of class $c$) of all parties}\\
        \hline
         $y(S)$ & \makecell[l]{The class of the time series generating $S$}\\
        \hline
      $D_{y(S)}$ & \makecell[l]{The subset of $D_S$ having the distances \\between $S$ and the time series of class $y(S)$} \\
      \hline
      \makecell[l]{$D_{S}^{\tau,L}$\\($D_{S,c}^{\tau,L}$)}  & \makecell[l]{The subsets of $D_S$ having the distances (of\\class $c$) not greater than the threshold $\tau$}\\
       \hline
        \makecell[l]{$D_{S}^{\tau,R}$\\($D_{S,c}^{\tau,R}$)}  & \makecell[l]{The subsets of $D_S$ having the distances (of\\class $c$) greater than the threshold $\tau$}\\
        \hline
      $Q_{IG}(S)$ & \makecell[l]{The quality of the shapelet $S$ measured as\\the maximum information gain}\\
        \hline
      $\{S_k\}_{k=1}^K$ & The set of shapelets of size $K$\\
      \hline
      $\mathcal{SC}$ & The set of the shapelet candidates\\
      \hline
      $\boldsymbol{X}_j$ ($\boldsymbol{X}^i_j$) & \makecell[l]{The feature vector of $T_j$ ($T^i_j$) transformed \\using the shapelets $\{S_k\}_{k=1}^K$}\\
        \hline
      $D$ ($D^i$) & \makecell[l]{The dataset (of $P_i$) transformed from $TD$ \\($TD^i$) using the shapelets}\\
      \hline
      $\langle x \rangle$ & The secretly shared value of $x$\\
      \hline
      $\langle x \rangle_i$ &The secret share held by the party $P_i$\\
      \hline
      $\boldsymbol{\gamma}_{A\subseteq D}$ & \makecell[l]{The vector of size $|D|$ indicating whether\\ the elements of $D$ are in $A$}\\
      \hline
      $\boldsymbol{\gamma}^i[j]$ & \makecell[l]{The value in the $j$-th entry of the indicating\\ vector $\boldsymbol{\gamma}^i$ held by $P_i$} \\
      \hline
      $\boldsymbol{\gamma}_L$/$\boldsymbol{\gamma}_R$/$\boldsymbol{\gamma}_c$ & \makecell[l]{The vectors of size $|D_S|$ indicating whether \\the elements of $D_S$ are in $D_{S}^{\tau,L}$/$D_{S}^{\tau,R}$/$D_{S,c}$}\\
      \hline
     \vspace{-2.2ex} & \\
     $\overline{D}_{S}$ ($\overline{D}_{S,c}$) & The mean of the distances in $D_S$ ($D_{S,c}$)\\
      \hline
      $Q_{F}(S)$ & \makecell[l]{The quality of the shapelet $S$ measured as \\the F-stat}\\    
         \bottomrule
    \end{tabular}
    \label{tab:notions}
\end{table}

	Intuitively, the simulator $\mathcal{S}$ must achieve the same effect in the ideal interaction as the adversary $\mathcal{A}$ achieves in the real interaction. In this paper, we identify the ideal functionality as the federated search of the high-quality shapelets, which is the kernel of the proposed FedST framework (see Sec.~\ref{fedst_kernel} in detail). Therefore, we contribute to \zyw{designing} secure and efficient protocols to achieve the functionality in the real FL scenario.

The federated setting in this paper is similar to the horizontal and cross-silo FL~\cite{kairouz2021advances,mammen2021federated}, because the data are horizontally partitioned across a few businesses and each of them has considerable but insufficient data. However, unlike the mainstream FL solutions that usually rely on a trust server \zy{(a.k.a. secure broker)}~\cite{huang2021personalized,zhang2020batchcrypt,marfoq2020throughput}, we remove this dependency considering that identifying such a party can cause additional costs~\cite{10.14778/3494124.3494125,tong2022hu}. Besides, the security definition we adopt is stricter than many existing FL works as mentioned above. Therefore, our setting \zyw{is} more practical but challenging.





\subsection{Shapelet Transformation}\label{TSC model}

\textit{Time series shapelets} are defined as representative subsequences that discriminate the classes. Denote $S = (s_1, \ldots, s_L)$ a shapelet generated from $TD=\{(T_j, y_j)\}_{j=1}^M$ and the length of $T_j$ is $N$, where $L \le N$. Let $T_j[s, l]$ denote the subseries of $T_j = (t_{j, 1}, \ldots, t_{j, N})$ that starts at the timestamp $s$ and has length $l$, i.e., 
\begin{equation}
	T_j[s, l] = (t_{j, s}, \ldots, t_{j, s + l - 1}), 1 \leq s \leq N - l + 1,
\end{equation}
the distance between the shapelet and the $j$-th time series is defined as the minimum Euclidean norm (ignore the square root) between $S$ and the $L$-length subseries of $T_j$, i.e.,
\begin{equation}
d_{T_j, S} = \mathop{\min}_{p \in \{1, \ldots, N - L + 1\}} ||S - T_j[p, L]||^2.
\label{eq:shapelet_dis}
\end{equation}

By definition, $d_{T_j, S}$ reflects the similarity between a localized shape of $T_j$ and $S$, which is a class-specific feature. The quality of $S$ can be measured by computing the distances to all series in $TD$, i.e., $D_S = \{d_{T_j, S}\}_{j=1}^M$, and evaluating the differences in \zyw{the} distribution of the distances between \zyw{the} class values $\{y_j\}_{j=1}^M$. The state-of-the-art method of shapelet quality measurement is to use the \textit{Information Gain (IG) with a binary strategy}~\cite{bostrom2017binary}. Each distance $d_{T_j, S} \in D_S$ is considered as a splitting threshold, denoted as $\tau$. The threshold is used to partition the dataset $D_S$ into $D_S^{\tau,L}$ and $D_S^{\tau,R}$, such that $D_S^{\tau,L} = \{d_{T_j, S}|d_{T_j, S} \le \tau\}_{j=1}^M$ and $D_S^{\tau,R} = D_S \setminus D_S^{\tau, L}$. The quality of $S$ is the maximum information gain among the thresholds, i.e.,
\begin{equation}
\begin{split}
Q_{IG}(S) &= \mathop{\max}_{\forall \tau}\ H(D_S) - (H(D_S^{\tau,L}) + H(D_S^{\tau,R})),
\end{split}\label{eq:IG}
\end{equation}
where 
\begin{equation}
	H(D) = -(p\log_2p + (1-p)\log_2(1-p)),
\end{equation}
$p = \frac{|D_{y(S)}|}{|D|}$is the fraction of samples in $D$ that belongs to the class of the sample generating $S$,
$y(S) \in \{c\}_{c=1}^C$ and
$D_{y(S)} = \{d_{T_j, S}| y_j = y(S)\}_{j=1}^M$.

In shapelet transformation, a set of candidates \zyw{is} randomly sampled from the possible subsequences of $TD$. After measuring the quality of all candidates, the $K$ subsequences with the highest quality are chosen as shapelets, which are denoted as $\{S_k\}_{k=1}^K$. The shapelets are used to transform the original dataset $TD$ into a \zyw{new} tabular dataset of $K$ features, where each attribute represents the distance between the shapelet and the original series, i.e., $D = \{(\mathbf{X}_j, y_j)\}_{j=1}^M$ where $\mathbf{X}_j = (d_{T_j, S_1}, \ldots, d_{T_j, S_K})$. The unseen series are transformed in the same way for prediction. $D$ can be used in conjunction with any classifier, such as the well-known intrinsically interpretable decision tree and logistic regression~\cite{interpretable-ML-book}.

\subsection{Secure Multiparty Computation}\label{mpc}
\textit{Secure multiparty computation (MPC)}~\cite{yao1982protocols} allows participants to compute a function over their inputs while keeping the inputs private. In this paper, we utilize the additive secret sharing scheme for MPC~\cite{damgaard2012multiparty} since it offers the protocols of the common arithmetic operations applicable to practical situations~\cite{chen2019secure,li2021privacy}. It \zyw{performs} in a field $\mathbb{Z}_q$ for a prime $q$. We denote a value $x \in \mathbb{Z}_q$ that is additively shared among parties as 
\begin{equation}
	\langle x \rangle = \{\langle x \rangle_0, \ldots, \langle x \rangle_{n-1}\},
\end{equation}
where $\langle x \rangle _i$ is a random \textit{share} of $x$ \zyw{held} by party $P_i$.

Suppose \zyw{that} $x$ is a private value of $P_i$. To secretly share $x$, $P_i$ randomly chooses \zyw{a \textit{share}} $\langle x \rangle_j \in \mathbb{Z}_q$ and sends it to $P_j(\forall j, j \neq i)$. Then, $P_i$ sets $\langle x \rangle_i = x - \sum_j \langle x \rangle_j \mod q$. To reconstruct $x$, all parties reveal their shares to compute $x = \sum_{i=0}^{n-1} \langle x \rangle_i \mod q$. For ease of exposition, we omit the modular operation in the rest of the paper.



Under the additive secret sharing scheme, a function $z = \mathit{f} (x, y)$ is computed by using \zyw{an} MPC protocol that takes $\langle x \rangle$ and $\langle y \rangle$ as input and outputs the \zyw{secretly} shared  $\langle z \rangle$. In this paper, we mainly use the following MPC protocols as building blocks: 

(a) \textit{Addition}: $\langle z \rangle = \langle x \rangle + \langle y \rangle$

(b) \textit{Multiplication}: $\langle z \rangle = \langle x \rangle \cdot \langle y \rangle$

(c) \textit{Division}: $\langle z \rangle =  \langle x \rangle / \langle y \rangle$

(d) \textit{Comparison}: $\langle z \rangle = \langle x \rangle \overset{?}{<} \langle y \rangle:\langle 1 \rangle:\langle 0 \rangle$

(e) \textit{Logarithm}: $\langle z \rangle = \log_2(\langle x \rangle)$

We refer \zyw{the reader} to \cite{beaver1991efficient,catrina2010secure,catrina2010improved,aly2019benchmarking} for the detailed implementation of the operations.

In addition, given the result $\langle b \rangle =  \langle x \rangle \overset{?}{<} \langle y \rangle:\langle 1 \rangle:\langle 0 \rangle$, the smaller one of \zyw{the} two values $ \langle x \rangle $, $\langle y \rangle$ can be securely assigned to $\langle z \rangle$, as:

(f) \textit{Assignment}: $\langle z\rangle =\langle b\rangle \cdot \langle x\rangle + (1-\langle b\rangle)\cdot \langle y\rangle$. 

With the assignment protocol, it is trivial to perform the \textit{maximum}, \textit{minimum}, and \textit{top-K} computation for a list of \zyw{secretly shared values} by sequentially comparing and swapping the adjacent elements in the list using the secure comparison and assignment protocols.

\section{Solution Overview}\label{solution_overview}

This section overviews our FL-enabled TSC framework, which is a key component of our FedTSC system~\cite{fedtsc} and is built based on the centralized shapelet transformation~\cite{hills2014classification,bostrom2017binary,bagnall2020tale}. We provide the framework overview in Sec.~\ref{fedst_framework}. Then, we identify the FedST kernel in Sec.~\ref{fedst_kernel}.

		\subsection{FedST Framework}\label{fedst_framework}

		Overall, FedST has two stages: (1) federated shapelet search; (2) federated data transformation and classifier training. The two stages are illustrated in Fig.~\ref{fig:fedst_framework}.
		
		In the first stage, all parties jointly search for the $K$ best shapelets $\{S_k\}_{k=1}^K$ from a candidate set $\mathcal{SC}$. 
		
		Note that $P_0$ requires the found shapelets to explain the shapelet-based features, so the shapelet candidates in $\mathcal{SC}$ are only generated by $P_0$ to ensure \zyw{that} the local time series of the participants cannot be accessed by the initiator. This may raise a concern that the shapelets will be missed if they do not occur in \zyw{$TD^0$}. Fortunately, since the high-quality shapelets are usually highly redundant in the training data, it is shown enough to find them by checking some randomly sampled candidates rather than all possible subsequences~\cite{bagnall2020tale,gordon2012fast}. Hence, it is feasible to generate $\mathcal{SC}$ by $P_0$ in our cross-silo setting where each business has considerable (but insufficient) data. We also verify this issue in Sec.~\ref{exp:accuracy} and \ref{exp:flexibility}.    
		
		In stage two, the time series data $TD^i$ in each party \zyw{are} transformed into the $K$ dimensional secretly shared tabular data as:
		\begin{equation}
			\langle D^i \rangle = \{(\langle\mathbf{X}_j^i\rangle, \langle y_j^i \rangle)\}_{j=1}^{M_i}, \forall i \in \{0,\ldots,n-1\},
		\end{equation}
		where 
		\begin{equation}
			\langle\mathbf{X}_j^i\rangle = (\langle d_{T_j^i, S_1}\rangle, \ldots, \langle d_{T_j^i, S_K}\rangle).
		\end{equation}
		
		Then, a standard classifier is built over the joint secretly shared data set $\langle D \rangle = \bigcup_{i=0}^{n-1}\langle D^i \rangle$.

	Note that there is always a trade-off between security and accuracy/interpretability in FL. To achieve a good balance, FedST ensures \zyw{that} only $P_0$ learns the shapelets and classifiers, while nothing else can be revealed to the parties. This degree of privacy has been shown practical by many FL systems~\cite{fu2021vf2boost,fu2022blindfl,wu13privacy}. Additionally, we illustrate in Sec.~\ref{dp-protect} that we can further enhance the security by incorporating differential privacy~\cite{dwork2014algorithmic}, guaranteeing that the revealed outputs leak limited information about the private training data.

	\subsection{FedST Kernel: Federated Shapelet Search}\label{fedst_kernel}
	
	The transformed data set $\langle D \rangle$ is a common tabular data set, with continuous attributes \zyw{that} can be used in conjunction with any standard classifier. Consequently, any classifier training protocol built for secretly shared data (e.g.,~\cite{mohassel2017secureml,abspoel2020secure,zheng2021cerebro,chen2019secure}) can be seamlessly integrated into our framework. Nevertheless, there exists no protocol that tackles the orthogonal problem of federated shapelet search and data transformation. Further, the data transformation is to compute the distances between each training series and shapelet, which is just a subroutine of the shapelet search. Thus, the key technical challenge within our FedST is to design secure and efficient protocols to achieve the \textit{federated shapelet search (Stage 1 in Fig.~\ref{fig:fedst_framework})}, which becomes the kernel part of FedST.
	
	Formally, we define the functionality of the federated shapelet search, $\mathcal{F}_{FedSS}$, as follows.
	\begin{figure}
			\centering
			\includegraphics[width=\linewidth]{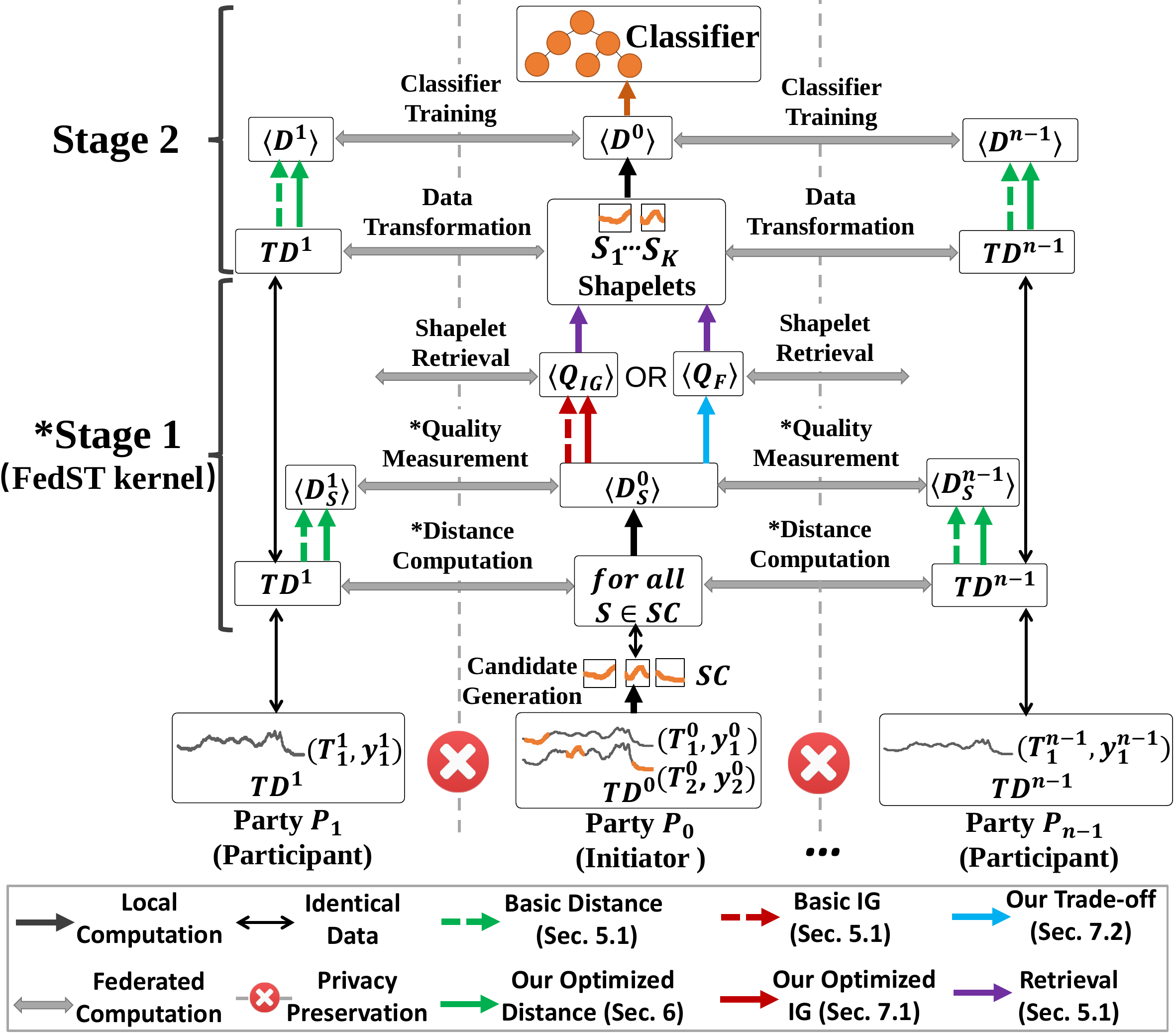}
			\caption{An illustration of the FedST framework.}
			\label{fig:fedst_framework}
		\end{figure}	
	
	\begin{definition}[Federated Shapelet Search, $\mathcal{F}_{FedSS}$]
		Given the time series datasets distributed over the parties, i.e., $TD^0$, \ldots, $TD^{n-1}$, and the shapelet candidates $\mathcal{SC}$ generated from $TD^0$, the goal of $\mathcal{F}_{FedSS}$ is to \zyw{find} the $K$ best shapelets $\{S_k|S_k \in \mathcal{SC}\}_{k=1}^K$ for $P_0$ by leveraging the distributed data sets.\label{definition:FedSS}
	\end{definition}

	\noindent

To realize ${\mathcal{F}_{FedSS}}$ under the security defined in Definition~\ref{definition:security}, a straightforward thought is to design security protocols by extending the centralized method to the FL setting using MPC. Following this, we present \textbf{${\Pi_{FedSS-B}}$} (Sec.~\ref{protocol_description}), the protocol that achieves our basic idea. We show \zyw{that} the protocol is \textit{secure} and \textit{effective} (Sec.~\ref{protocol_discussion}), but we identify that {it suffers from \textit{low efficiency} due to the high communication overhead incurred by MPC and the failure of the pruning techniques due to the security issue (Sec.~\ref{protocol_bottleneck}).} 

To tackle the efficiency issue, we propose \textit{secure acceleration techniques tailored for the FL setting} that dramatically boost the protocol efficiency by optimizing the two bottlenecked processes of $\Pi_{FedSS-B}$, i.e., the \textit{distance computation} (Sec.~\ref{Distance_Acceleration}) and the \textit{quality measurement} (Sec.~\ref{measurement_acceleration}). \zyw{Experimental} results show that each of these techniques is \textit{individually effective} and they together contribute to \textbf{three orders of magnitude of speedup} (Sec.~\ref{exp:efficiency}).

Besides, since the evaluation of each shapelet candidate is in a randomized order and independent \zyw{of} the others, FedSS can perform in an anytime fashion~\cite{bagnall2020tale,gordon2012fast}. That is, the user announces a time contract, so that the evaluation stops once the running time exceeds the contract, and only the assessed candidates are considered in the following steps. Since this strategy relies only on the publicly available running time, it is feasible in the FL setting~\cite{fedtsc} to flexibly balance the accuracy and efficiency. We verify this issue in Sec.~\ref{exp:flexibility}.

	\section{Basic Protocol $\Pi_{FedSS-B}$}\label{fedss_bs}
	We now introduce the basic protocol $\Pi_{FedSS-B}$, which is extended from the centralized shapelet search using MPC to protect the intermediate information (Sec.~\ref{protocol_description}). We discuss the protocol in terms of security, effectiveness and efficiency in Sec.~\ref{protocol_discussion}, and analyze the bottlenecks of the protocol in Sec.~\ref{protocol_bottleneck}.
	
	\subsection{Protocol Description}\label{protocol_description}
	
	
	$\Pi_{FedSS-B}$ is outlined in Algorithm~\ref{alg:fedss_bs}.  The parties jointly assess the quality of each candidate and then select the $K$ best as the shapelets. The algorithm performs in three steps. First, the parties compute the distance between the samples and each candidate (Lines 2-8). Second, the parties evaluate the quality of the candidate over the secretly shared distances and labels (Lines 9). Finally, the parties jointly retrieve the $K$ candidates with the highest quality and reveal the shares of the indices to $P_0$ to recover the selected shapelets (Lines 10-11). These three steps are described as follows.

	{\setlength{\parindent}{0cm}
		\textbf{Distance Computation.} Since the candidates are locally generated by $P_0$, the distance between the samples of $P_0$ and the candidates can be locally computed. After that, $P_0$ secretly shares the results to enable the subsequent steps (Lines 3-5).
	}
	
	To compute the distances between the samples of each participant $P_i$ and the candidates (Lines 6-8), the MPC operations have to be adopted. For example, to compute $d_{T_j^i, S}$, $P_i$ and $P_0$ secretly share $T_j^i$ and $S$ respectively. Next, the parties jointly compute each Euclidean norm $\langle ||S, T^i_j[p, L]||^2 \rangle$ using MPC. At last, the parties jointly determine the shapelet distance $\langle d_{T_j^i, S} \rangle$ by Eq.~\ref{eq:shapelet_dis} using the secure minimum operation (see Sec.~\ref{mpc}).
	
	\begin{algorithm}[t]
		\normalem
		\caption{Basic Protocol $\Pi_{FedSS-B}$}
		\label{alg:fedss_bs}
		\SetKwData{Or}{\textbf{or}}
		\DontPrintSemicolon
		\KwIn {$TD^i=\{( T_j^i , y_j^i )\}_{j=1}^{M_i}$, $i = 0,\ldots,n-1$: local datasets\\ \ \ \ \ \ \ \
			\ $\mathcal{SC}$: A set of shapelet candidates locally generated by $P_0$
			\\ \ \ \ \ \ \ \ \ \ \ $K$: the number of shapelets}
		\KwOut {$\{S_k\}_{k=1}^{K}$: shapelets revealed to $P_0$}
		
		
		\For{$S \in \mathcal{SC}$}{
			\For{$i \in \{0,\ldots,n-1\}$}{
				\If{$i == 0$}{
					\For{$j \in \{1, \ldots, M_0\}$}{$P_0$ locally computes $d_{T^0_j, S}$ and secretly shares the result among all parties \;}
				}
				\Else{
					\For{$j \in \{1, \ldots, M_i\}$}{All parties jointly compute $\langle d_{T^i_j, S} \rangle$\;}
					
				}		
				
			}
			
			All parties jointly compute the quality $\langle Q_{IG}(S) \rangle$ over the secretly shared distances and labels \;
			All parties jointly find the $K$ candidates with the highest quality and reveal the indices $\{\langle I_k \rangle \}_{k=1}^K$ to $P_0$ \;
			
		}
		\Return{$\{S_k = \mathcal{SC}_{I_k}\}_{k=1}^K$}
	\end{algorithm}
	

	{\setlength{\parindent}{0cm}
		\textbf{Quality Measurement.} Based on Eq.~\ref{eq:IG}, to compute the IG quality of $S \in \mathcal{SC}$ (Line 9), we need to {securely partition the dataset $D_S$ using each threshold $\tau$ and compute the number of samples belonging to each class $c\ (c \in \{1, \ldots,C\})$ for $D_S$, $D_S^{\tau,L}$, and $D_S^{\tau, R}$.} We achieve it over the secretly shared distances and labels by leveraging the \textit{indicating vector} defined as follows.
	}
	
	\begin{definition}[Indicating Vector]
		Given a dataset $D = \{x_j\}_{j=1}^M$ and a subset $A \subseteq D$, we define the indicating vector of $A$, denoted as $\bm{\gamma}_{A \subseteq D}$, as a vector of size $M$ whose $j$-th ($j\in\{1,\ldots,M\}$) entry represents whether $x_j$ is in $A$, i.e., $\boldsymbol{\gamma}_{A \subseteq D}[j] = 1$ if $x_j \in A$, and $0$ otherwise.
	\end{definition}

	For example, for $D = \{x_1, x_2, x_3\}$ and $A = \{x_1,x_3\}$, the indicating vector of $A$ is $\boldsymbol{\gamma}_{A \subseteq D} = (1, 0, 1)$. Suppose that $\boldsymbol{\gamma}_{A_1 \subseteq D}$ and $\boldsymbol{\gamma}_{A_2 \subseteq D}$ are the indicating vectors of $A_1$ and $A_2$, respectively, we have
	\begin{equation}
		 \boldsymbol{\gamma}_{A_1 \subseteq D} \cdot \boldsymbol{\gamma}_{A_2 \subseteq D} = |A_1 \cap A_2|,
	\end{equation}
	where $|A_1 \cap A_2|$ is the cardinality of $A_1 \cap A_2$. Specifically, we have $\boldsymbol{\gamma}_{A_1 \subseteq D} \cdot \boldsymbol{1} = |A_1|$.

	With the indicating vector, we securely compute $\langle Q_{IG}(S) \rangle$ as follows.
	
	At the beginning, $P_0$ generates a vector of size $C$ to indicate the class of $S$, i.e.,
	\begin{equation}
		\boldsymbol{\gamma}_{y(S)} = \boldsymbol{\gamma}_{\{y(S)\} \subseteq \{c\}_{c=1}^C},
	\end{equation}
	 and secretly shares the vector among all parties. 
	 
	 Next, for each splitting threshold
	 \begin{equation}
	 	\langle \tau \rangle \in \bigcup_{i=0}^{n-1} \{\langle d_{T_j^i,S} \rangle\}_{j=1}^{M_i},
	 \end{equation}
	 the parties jointly compute the secretly shared vector 
	 \begin{equation}
	 \begin{split}
	 \langle \boldsymbol{\gamma}_L \rangle &= \langle \boldsymbol{\gamma}_{D_S^{\tau,L} \subseteq D_S} \rangle,\\
	 \langle \boldsymbol{\gamma}_R \rangle &= \langle \boldsymbol{\gamma}_{D_S^{\tau,R} \subseteq D_S} \rangle = \boldsymbol{1} - \langle \boldsymbol{\gamma}_L \rangle,
	 \end{split}
	 	\end{equation}
	 where
	 \begin{equation}
	 	\langle \boldsymbol{\gamma}_{D_S^{\tau,L} \subseteq D_S}[j] \rangle = \langle d_{T_j^i, S}\rangle \overset{?}{<} \langle \tau \rangle,\ j \in \{1, \ldots, M\}.
	 \end{equation}
	 Meanwhile, each party $P_i$ secretly shares the vector $\boldsymbol{\gamma}^{i}_{TD^{i}_{c} \subseteq TD^{i}}$ to indicate its samples that belong to each class $c$. Denote the indicating vectors of all parties as 
	 \begin{equation}
	 	\langle \boldsymbol{\gamma}_{c} \rangle = (\langle \boldsymbol{\gamma}^0_{TD^0_{c} \subseteq TD^0}\rangle, \ldots, \langle \boldsymbol{\gamma}^{n-1}_{TD^{n-1}_{c} \subseteq TD^{n-1}}\rangle),
	 \end{equation}
	which indicates the samples in $D_S$ that belong to class $c$, i.e.,
	\begin{equation}
		\langle \boldsymbol{\gamma}_{c} \rangle = \langle \boldsymbol{\gamma}_{TD_{c} \subseteq TD} \rangle = \langle \boldsymbol{\gamma}_{D_{S,c} \subseteq D_S} \rangle.
	\end{equation}
	As such, the parties compute the following statistics using MPC:
	\begin{equation}
	\begin{split}
	\langle |D_S^{\tau, L}| \rangle &= \langle \boldsymbol{\gamma}_L \rangle \cdot \boldsymbol{1},\\
	\langle |D_S^{\tau, R}| \rangle &= |D_S| - \langle |D_S^{\tau, L}| \rangle,\\
	\langle |D_{S,y(S)}| \rangle &= \langle \boldsymbol{\gamma}_{y(S)} \rangle \cdot (\langle \boldsymbol{\gamma}_{1} \rangle \cdot \boldsymbol{1}, \ldots,\langle \boldsymbol{\gamma}_{C} \rangle \cdot \boldsymbol{1}),\\
	\langle |D_{S, y(S)}^{\tau, L}| \rangle &= \langle \boldsymbol{\gamma}_{y(S)} \rangle \cdot (\langle \boldsymbol{\gamma}_{{1}} \rangle \cdot \langle \boldsymbol{\gamma}_L \rangle, \ldots, \langle \boldsymbol{\gamma}_{{C}} \rangle \cdot \langle \boldsymbol{\gamma}_L \rangle),\\
	\langle |D_{S, y(S)}^{\tau, R}| \rangle &= \langle \boldsymbol{\gamma}_{y(S)} \rangle \cdot (\langle \boldsymbol{\gamma}_{{1}} \rangle \cdot \langle \boldsymbol{\gamma}_R \rangle, \ldots, \langle \boldsymbol{\gamma}_{{C}} \rangle \cdot \langle \boldsymbol{\gamma}_R \rangle).
	\end{split}\label{eq:IG_stats}
	\end{equation}
	Given the statistics in Eq.~\ref{eq:IG_stats} and the public value $|D_S|=M$, the parties can jointly compute $\langle Q_{IG}(S) \rangle$ by Eq.~\ref{eq:IG}.

	{\setlength{\parindent}{0cm}
		\textbf{Shapelet Retrieval.} Given the quality of the candidates in secret shares, the parties jointly retrieve the indices of the $K$ best shapelets (Line 10) by securely comparing the adjacent quality values and then swapping the values and the corresponding indices based on the comparison results (see Sec.~\ref{mpc}). The indices are output to $P_0$ to recover the jointly selected shapelets (Line 11).
	}

	\subsection{Protocol Discussion}\label{protocol_discussion}
	This section analyzes $\Pi_{FedSS-B}$ in terms of security, effectiveness, and efficiency.
	
	{\setlength{\parindent}{0cm}
		\textbf{Security.} The security of $\Pi_{FedSS-B}$ is guaranteed by the following \zyw{theorem}:
	}
	\begin{theorem}
		$\Pi_{FedSS-B}$ is secure under the security defined in Definition~\ref{definition:security}.
		\label{theorem:fedst_basic_security}
	\end{theorem}
	\begin{proof}[Proof Sketch]
		In  $\Pi_{FedSS-B}$, all joint computations are executed using MPC. With the indicating vector, the secure computations are data-oblivious.  An adversary learns no additional information. The security follows.
	\end{proof}
	{\setlength{\parindent}{0cm}
		\textbf{Effectiveness.} We discuss the protocol effectiveness in terms of classification accuracy.
		$\Pi_{FedSS-B}$ is directly extended from the centralized approach by using the secret-sharing-based MPC operations, which have considerable computation precision~\cite{catrina2010secure,catrina2010improved,aly2019benchmarking}. Therefore, it is expected that the accuracy of FedST has no difference from the centralized approach. The experiment results in Sec.~\ref{exp:accuracy} validate this issue.
	}
	
	\begin{figure}[t]
		\centering
		\includegraphics[width=.65\linewidth]{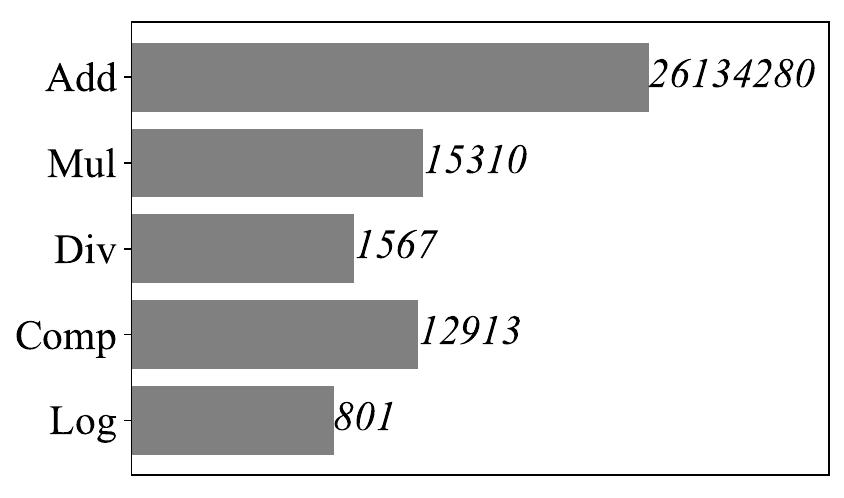}
		\caption{Throughputs (\#operations per second) of different MPC operations executed by three parties. Secure addition is much more efficient than the others because it is executed without communication~\cite{catrina2010secure}.}
		\label{fig:mpc_throughputs}
	\end{figure}

	{\setlength{\parindent}{0cm}
		\textbf{Efficiency.} As shown in Fig.~\ref{fig:mpc_throughputs}, the secret-sharing-based MPC is usually bottlenecked by communication rather than computation. Therefore, \textit{it is more indicative to analyze the efficiency by considering the complexity of only the interactive operations}, including secure multiplication, division, comparison and logarithm operations. We follow this metric for efficiency analysis in the paper.
		
	}
	
	$\Pi_{FedSS-B}$ in Algorithm~\ref{alg:fedss_bs} takes $O(|\mathcal{SC}| \cdot MN^2)$ for distance computation. The quality measurement has a complexity of $O(|\mathcal{SC}| \cdot M^2)$. Securely finding the top-$K$ candidates has a complexity of $O(|\mathcal{SC}| \cdot K)$. Since $K$ is usually a small constant, the total complexity of $\Pi_{FedSS-B}$ can be simplified as $O(|\mathcal{SC}| \cdot(MN^2  + M^2))$.
	
	
	
		\begin{figure}
		\centering
		\includegraphics[width=\linewidth]{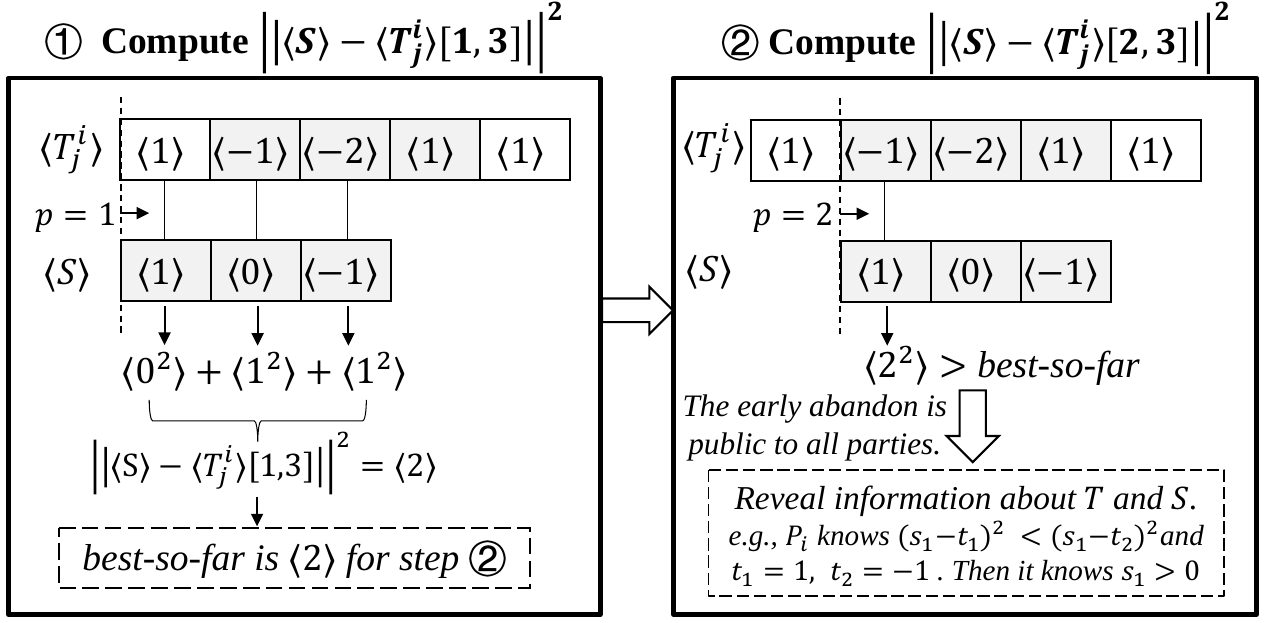}
		\caption{Illustration of the Euclidean norm pruning and its information disclosure.}
		\label{fig:abandon}
	\end{figure}
	
	\subsection{Bottleneck Analysis}\label{protocol_bottleneck}
	
	As discussed in Sec.~\ref{protocol_discussion}, $\Pi_{FedSS-B}$ is secure and effective \zyw{in enabling the} federated shapelet search. However, the basic protocol has expensive time cost in the FL setting for both \textit{distance computation} and \textit{quality measurement} steps, which bottleneck the efficiency of the protocol. Two \zyw{main} reasons are as \zyw{follows}.
	
	{\setlength{\parindent}{0cm}
	\textbf{\textit{Reason I. Heavy Communication Overhead.}} As discussed in Sec.~\ref{protocol_discussion}, $\Pi_{FedSS-B}$ takes $O(|\mathcal{SC}|\cdot MN^2)$ and $O(|\mathcal{SC}\cdot|M^2)$ expensive interactive operations to compute the distance and measure the quality for all candidates, which dominate the complexity. 
	Therefore, {the efficiency of $\Pi_{FedSS-B}$ is bottlenecked by the first two steps, i.e., distance computation and quality measurement.}
}

	{\setlength{\parindent}{0cm}
	\textbf{\textit{Reason II. Failure of Acceleration Techniques.}} Even using only local computation, repeatedly computing the distance and quality for all candidates \zyw{is} time-consuming~\cite{ye2011time}. To \zyw{address} this, existing studies propose pruning strategies for acceleration~\cite{mueen2011logical,keogh2006lb_keogh,ye2011time,rakthanmanon2012searching}. Unfortunately, {the pruning techniques are inevitably \textit{data-dependent}, which violates the security of Definition~\ref{definition:security} that requires the federated computation oblivious.} Thus, we have to abandon these acceleration strategies in the FL setting. We show the security issue in Theorem~\ref{theorem:distance_pruning_security} and Theorem~\ref{theorem:IG_pruning_security}.
}

	\begin{theorem}
		Protocol $\Pi_{FedSS-B}$ is insecure under the security defined in Definition~\ref{definition:security} if using the Euclidean norm pruning strategies proposed in~\cite{keogh2006lb_keogh} and \cite{rakthanmanon2012searching}.
		\label{theorem:distance_pruning_security}
	\end{theorem}
	\begin{proof}[Proof Sketch]
		Fig.~\ref{fig:abandon} illustrates the Euclidean norm pruning. The basic idea is to maintain a best-so-far Euclidean norm at each timestamp $p \in \{1,\ldots,N-L_S+1\}$, and incrementally compute the sum of the squared differences between each pair of data points when computing $||S-T^i_j[p,L_S]||^2$ (left). Once the sum exceeds the best-so-far value, the current norm computation can be pruned (right). In the FL setting, although we can incrementally compute the sum and compare it with the best-so-far value using MPC, the comparison result must be disclosed when determining the pruning, which cannot be achieved by the simulator $\mathcal{S}$ that attacks the ideal interaction $\mathcal{F}_{FedSS}$ in Definition~\ref{definition:FedSS}. The security is violated.
	\end{proof}
	
	Similarly, \zyw{we} have the following theorem.
	
	\begin{theorem}
		Protocol $\Pi_{FedSS-B}$ is insecure under the security defined in Definition ~\ref{definition:security} if using the IG quality pruning strategies proposed in~\cite{ye2011time} and \cite{mueen2011logical}.
		\label{theorem:IG_pruning_security}
	\end{theorem}
	
	{\setlength{\parindent}{0cm}
		We omit the proof because it is similar to the proof of Theorem~\ref{theorem:distance_pruning_security}.

		\textbf{Optimizations.} To remedy the efficiency issue of the protocol $\Pi_{FedSS-B}$, we propose \textit{acceleration methods tailored for the FL setting} to improve the efficiency of the {distance computation} and the {quality measurement steps.} }
		
		For distance computation, we propose to speed up the bottlenecked Euclidean norm computation based on \textit{a novel secure dot-product protocol} (Sec.~\ref{Distance_Acceleration}). 
		
		For quality measurement, we first propose a secure \textit{sorting-based acceleration} to reduce the duplicated interactive operations for computing IG (Sec.~\ref{sorting_based_acceleration}). Then, we propose to \textit{tap an alternative F-stat measure} to further improve the efficiency with comparable accuracy (Sec.~\ref{trade_off}). 
		
		\zyw{The experiments} show that each of these three techniques is individually effective and they together brings up to \textit{three orders of magnitude of speedup} to $\Pi_{FedSS-B}$. \zyw{Furthermore}, compared to our well-optimized \zyw{IG-based protocol}, the F-stat-based method in Sec.~\ref{trade_off} gives 1.04-8.31x of speedup while guaranteeing \textit{no statistical difference} in TSC accuracy. (Sec.~\ref{exp:efficiency}).

	\section{Shapelet Distance Acceleration}\label{Distance_Acceleration}
	In $\Pi_{FedSS-B}$, the distance between a candidate $S$ and the $M-M_0$ samples $T^i_j(\forall j, i \neq 0)$ is straightforwardly computed using MPC. Based on Eq.~\ref{eq:shapelet_dis}, the interactive operations used include: 
	
	\begin{enumerate}
		\item $L_S(N - L_S + 1)(M-M_0)$ pairwise multiplications for the Euclidean norm;
		\item $(N - L_S + 1)(M - M_0)$ times of both comparisons and assignments \zyw{to find} the minimum. 
	\end{enumerate}

	Because the shapelet length $L_S$ is up to $N$ where $N >> 1$, the efficiency is dominated by the Euclidean norm. Thus, {it is necessary to accelerate the distance computation by improving the efficiency of the bottlenecked \textit{Euclidean norm}}.

	The work of~\cite{ioannidis2002secure} proposes a two-party dot-product protocol (as Algorithm~\ref{alg:original_dp}), \zyw{which we find to be efficient both in} computation and communication for the calculation between one vector and many others. It motivates us that we can compute the Euclidean norm between a candidate $S$ and the total $(N - L_S + 1)(M-M_0)$ subseries of the participants using the dot-product protocol. Unfortunately, the protocol in Algorithm~\ref{alg:original_dp} (denoted as the raw protocol) has weak security that violates Definition~\ref{definition:security}. 
	
	To overcome the limitation, we propose ${\Pi_{DP}}$, \textit{a secure dot-product protocol} that enhances the raw protocol using MPC. We prove that this novel protocol not only follows the security of Definition~\ref{definition:security}, but also effectively accelerates the Euclidean norm. We describe the acceleration method in Sec.~\ref{DP-based-ED}. Then, we analyze the security deficiency of the raw protocol and propose our $\Pi_{DP}$ in Sec.~\ref{security_analysis_and_enhancement}.

	\begin{algorithm}[t]
		\caption{The Two-Party Dot-Product Protocol of~\cite{ioannidis2002secure} }
		\label{alg:original_dp}
		\SetKwData{Or}{\textbf{or}}
		\DontPrintSemicolon
		\KwIn {$\boldsymbol{x} \in \mathbb{R}^L$ from $P_0$; $\boldsymbol{y} \in \mathbb{R}^L$  from $P_i$ ($i \in \{1,\ldots,n-1\}$)}
		\KwOut {$\beta$ to $P_0$ and $\alpha$ to $P_i$, satisfying $\beta - \alpha = \boldsymbol{x}^T \cdot \boldsymbol{y}$}
		
		
		\textbf{Party $\boldsymbol{P_0}$} randomly chooses $\boldsymbol{Q}$, $r$, $\boldsymbol{f}$, $R_1$, $R_2$, $R_3$, $\mathbf{x}_i$ ($i \in \{1,\ldots, d\}, i \neq r$) and creates $\mathbf{X}$. Then, it computes $b $, $\boldsymbol{U}$, $\boldsymbol{c}$, $\boldsymbol{g}$, and sends $\boldsymbol{U}$, $\boldsymbol{c}$, $\boldsymbol{g}$ to $P_i$ \;
		\textbf{Party $\boldsymbol{P_i}$} randomly chooses $\alpha$, creates $\boldsymbol{y}^\prime$, computes and sends to $P_0$ the value $a$, $h$ \;
		\textbf{Party $\boldsymbol{P_0}$} computes $\beta$ \;
		
	\end{algorithm}

	\subsection{Dot-Product-based Euclidean Norm}\label{DP-based-ED}
	
	Given two vectors $\boldsymbol{x} \in \mathbb{R}^L$ from $P_0$ and $\boldsymbol{y} \in \mathbb{R}^L$ from $P_i$, Algorithm~\ref{alg:original_dp} computes the dot-product $\boldsymbol{x}\cdot\boldsymbol{y}$ as follows. 
	
	\noindent
	\textbf{(i)} $P_0$ chooses a random matrix $\boldsymbol{Q} \in \mathbb{R}^{d \times d}  (d \ge 2)$, a random value $r \in \{1, \ldots, d\}$, a random vector $\boldsymbol{f} \in \mathbb{R}^{L+1}$ and three random values $R_1, R_2, R_3$, and selects $s - 1$ random vectors 
	\begin{equation}
		\mathbf{x}_i \in \mathbb{R}^{L+1},\ i \in \{1,\ldots, r-1, r+1, \ldots, d\}.
	\end{equation}
	Next, it creates a matrix 
	\begin{equation}
		\mathbf{X} \in \mathbb{R}^{d \times (L + 1)},
	\end{equation}
	whose $i$-th row ($i \neq r $) is $\mathbf{x}_i$ and $r$-th row is
	\begin{equation}
		\boldsymbol{x}^{\prime T} = (x_1, \ldots, x_L, 1).
	\end{equation}
	Then, $P_0$ locally computes 
	\begin{equation}
			b = \sum_{j=1}^d \boldsymbol{Q}_{j,r},
	\end{equation}
	\begin{equation}
		\boldsymbol{U} = \boldsymbol{Q} \cdot \boldsymbol{X},
	\end{equation}
	\begin{equation}
		\boldsymbol{c} = \sum_{i \in \{1,\ldots,d\}, i \neq r}(\boldsymbol{x}_i^T \cdot \sum_{j=1}^{d}\boldsymbol{Q}_{j,i}) + R_1R_2\boldsymbol{f}^T,
	\end{equation}	
	\begin{equation}	
		\boldsymbol{g} = R_1R_3\boldsymbol{f}.
	\end{equation}
	Finally, it sends $\boldsymbol{U}$, $\boldsymbol{c}$, $\boldsymbol{g}$ to $P_i$ (Line 1);

	\noindent
	\textbf{(ii)} $P_i$ chooses a random value $\alpha$ to generate 
	\begin{equation}
		\boldsymbol{y}^\prime = (y_1, \ldots, y_L, \alpha)^T.
	\end{equation}
	Next, it computes and sends to $P_0$ two scalars $a$ and $h$ (Line 2), as:
	\begin{equation}
	a = \sum_{j=1}^{d}\boldsymbol{U}_{j}\cdot\boldsymbol{y}^\prime - \boldsymbol{c}\cdot\boldsymbol{y}^\prime,
	\end{equation}
	\begin{equation}
	h = \boldsymbol{g}^T\cdot\boldsymbol{y}^\prime.
	\end{equation}
	
	\noindent
	\textbf{(iii)} $P_0$ locally computes $\beta$ (Line 3) as: 
	\begin{equation}
		\beta = \frac{a}{b} + \frac{hR_2}{bR_3}.
	\end{equation}
	
	Given $\beta$ and $\alpha$, the result satisfies $\boldsymbol{x}^T\cdot\boldsymbol{y} = \beta - \alpha$.


	The \textit{Euclidean norm computation} in our federated shapelet search can benefit from the above protocol, since each  $||S, T^i_j[p, L_S]||^2$  can be represented as
	\begin{equation}
	\begin{split}
	||S, T^i_j[p, L_S]||^2 =& \sum_{p^\prime=1}^{L_S} (s_{p^\prime})^2 + \sum_{p^\prime=1}^{L_S} (t_{p^\prime+p-1})^2 \\&+ 2\sum_{p^\prime=1}^{L_S} s_{p^\prime}t_{p^\prime+p-1},
	\end{split}\label{eq:Euclidean_split}
	\end{equation}
	where the term
	\begin{equation}
		z = \sum_{p^\prime=1}^{L_S} s_{p^\prime}t_{p^\prime+p-1} = \boldsymbol{S}^T\cdot\boldsymbol{T}^i_j[p,L_S]\label{eq:Euclidean_DP}
	\end{equation}
	can be computed by $P_0$ and $P_i$ jointly executing the protocol to get $\beta$ and $\alpha$, respectively. The terms $\sum_{p^\prime=1}^{L_S} (s_{p^\prime})^2$ and $\sum_{p^\prime=1}^{L_S} (t_{p^\prime+p-1})^2$ can be locally computed by the two parties. To this end, all parties aggregate the three terms in secret shares using non-interactive secure addition. 
	
	Using the above dot-product protocol, the
	total communication cost for the $(N-L_S+1)(M-M_0)$ Euclidean norm between $S$ and the subseries of the participants is reduced from $O(L_S(N-L_S+1)(M-M_0))$ to $O(L_S) + O((N - L_S + 1)(M-M_0))$.

\zy{Note that following Eq.~\ref{eq:Euclidean_split}, the key to computing $d_{T^i_j,S}$ is to securely compute Eq.~\ref{eq:Euclidean_DP} for $p$ from $1$ to $N-L_S+1$, which is equivalent to the standard 1-D convolution operation widely used in deep learning~\cite{lecun2015deep-DL-Book}. Although federated convolution computation has been extensively studied in recent years~\cite{bian2021apas-conv1,lee2022low-conv2,rivinius2023convolutions-conv3}, existing methods focus on the acceleration of many convolutions between a large number of inputs (i.e., data and filters) of fixed sizes by packing or parallelism, which corresponds to the computational workload of the well-known convolutional neural networks~\cite{lecun2015deep-DL-Book}.} 

\zy{However, in the shapelet search scenario, the shapelet candidates have various and commonly different lengths, which can hardly be dealt with using those federated convolution protocols. Furthermore, unlike general MPC scenarios where the input is in secret shares, our input vectors, that is, the time series $\boldsymbol{T}^i_j$ and the candidate $\boldsymbol{S}$, are held locally by the two parties. This special case allows us to design a secure and efficient dot-product protocol without the costly offline pre-processing required by the aforementioned approaches. We will elaborate on our novel protocol in the next section.}
 

	\begin{algorithm}[t]
		\caption{Secure Dot-Product Protocol $\Pi_{DP}$ (Ours)}
		\label{alg:secure_dp}
		\SetKwData{Or}{\textbf{or}}
		\DontPrintSemicolon
		\KwIn {$\boldsymbol{x} \in \mathbb{R}^L$  from $P_0$; $\boldsymbol{y} \in \mathbb{R}^L$ from $P_i$ ($i \in \{1,\ldots,n-1\}$)}
		\KwOut {$\langle z \rangle$ secretly shared by all parties, satisfying $z=\boldsymbol{x}^T \cdot \boldsymbol{y}$}
		
		\textbf{Party $\boldsymbol{P_0}$} and \textbf{Party $\boldsymbol{P_i}$} represent each element of their input vectors as \zyw{a} fixed-point number encoded in $\mathbb{Z}_q$ as used in MPC \\
		\textbf{Party $\boldsymbol{P_0}$} independently and randomly chooses each value of $\boldsymbol{Q}$, $\boldsymbol{f}$, $R_1$, $R_2$, $R_3$, $\mathbf{x}_i$ ($i \in \{1,\ldots, d\}, i \neq r$) from $\mathbb{Z}_q$, $r \in {\{1,\ldots, d\}}$, creates $\mathbf{X}$, computes $b$, $\boldsymbol{U}$, $\boldsymbol{c}$, $\boldsymbol{g}$, and sends $\boldsymbol{U}$, $\boldsymbol{c}$, $\boldsymbol{g}$ to $P_i$. \\
		\textbf{Party $\boldsymbol{P_i}$} randomly chooses $\alpha \in \mathbb{Z}_q$, creates $\boldsymbol{y}^\prime$, and computes the value $a$, $h$. Then, \textit{$P_i$ sends only $h$ to $P_0$} \\
		\textbf{All Parties} \textit{jointly compute $\langle z \rangle = \langle \beta \rangle - \langle \alpha \rangle = \langle \frac{1}{b} \rangle\cdot\langle a \rangle + \langle \frac{hR_2}{bR_3} \rangle- \langle \alpha \rangle$ using MPC} \\	
	\end{algorithm}
	
	\subsection{Security Analysis and Enhancement}\label{security_analysis_and_enhancement}
	Although the protocol in Algorithm~\ref{alg:original_dp} benefits the efficiency of the distance computation, {it is unavaliable due to the security issue.}

	\begin{theorem}
		The protocol of Algorithm~\ref{alg:original_dp} is insecure under the security defined in Definition~\ref{definition:security}.
	\end{theorem}
	\begin{proof}[Proof Sketch]
		Consider an adversary $\mathcal{A}$ that attacks $P_0$. By executing the raw protocol in Algorithm~\ref{alg:original_dp}, $\mathcal{A}$ receives the messages $a$ and $h$. For ease of exposition, we represent the matrix $\boldsymbol{U}$ as a row of the column vectors, i.e., 
		\begin{equation}
			\boldsymbol{U} = (\boldsymbol{u}_1, \ldots, \boldsymbol{u}_{L+1}),
		\end{equation}
		where
		\begin{equation}
			\boldsymbol{u}_{i} = (\boldsymbol{U}_{1,i},\ldots,\boldsymbol{U}_{d,i})^T, i \in \{1,\ldots,L+1\}.
		\end{equation}
		Denote 
		\begin{equation}
			\boldsymbol{c} = (c_1, \ldots, c_{L+1})
		\end{equation}
		and
		\begin{equation}
			\boldsymbol{g}^T = (g_1, \ldots, g_{L+1}).
		\end{equation}
		Recall that $\boldsymbol{v}^\prime = (\boldsymbol{v}^T, \alpha)^T$. Thus, it has
		\begin{equation}
		\begin{split}
		a = \sum_{j=1}^{d}\boldsymbol{U}_j \cdot \boldsymbol{v}^\prime - \boldsymbol{c} \cdot \boldsymbol{v}^\prime = \boldsymbol{e}_1^T \cdot \boldsymbol{v} + w\alpha,
		\end{split}\label{eq:a_alpha}
		\end{equation}
		\begin{equation}
		h = \boldsymbol{g}^T \cdot \boldsymbol{v}^\prime = \boldsymbol{e}_2^T \cdot \boldsymbol{v} + g_{L+1}\alpha,
		\label{eq:h_alpha}
		\end{equation}
		where
		\begin{equation}
			\boldsymbol{e}_1 = (\sum\boldsymbol{u}_{1} + c_{1}, \ldots, \sum\boldsymbol{u}_{L} + c_{L})^T,
		\end{equation}
		\begin{equation}
			w = (\sum\boldsymbol{u}_{L+1} + c_{L+1}),
		\end{equation}
		\begin{equation}
			\boldsymbol{e}_2 = (g_1, \ldots, g_L)^T.
		\end{equation}
		Based on Eq.~\ref{eq:a_alpha}-\ref{eq:h_alpha}, $\mathcal{A}$ knows that 
		\begin{equation}
			g_{L+1}a - wh = (g_{L+1}\boldsymbol{e}_1^T - w\boldsymbol{e}_2^T) \cdot \boldsymbol{v},
		\end{equation}
		where $g_{L+1}\boldsymbol{e}_1^T - w\boldsymbol{e}_2^T$ is created locally in $P_0$. Obviously, the probability distribution of $g_{L+1}a - wh$ \zyw{depends} on the private data $\boldsymbol{v}$, which cannot be simulated by any \zyw{simulator} $\mathcal{S}$.
	\end{proof}
	{\setlength{\parindent}{0cm}
		\textbf{Our novel protocol $\boldsymbol{\Pi_{DP}}$.} To securely achieve the acceleration, we propose $\Pi_{DP}$, \textit{a novel dot-product protocol that follows the security} in Definition~\ref{definition:security}. The basic idea is to enhance the security of Algorithm~\ref{alg:original_dp} using MPC and the finite field arithmetic}. This solution is simple \zyw{yet} rather effective in terms of both security and efficiency.

	$\Pi_{DP}$ is presented in Algorithm~\ref{alg:secure_dp}. It has three differences to the raw protocol: 
	
	\begin{enumerate}
		\item $P_0$ and $P_i$ represent each element of their input vectors as \zyw{a} fixed-point number encoded in $\mathbb{Z}_q$ as used in MPC~\cite{catrina2010secure,catrina2010improved,aly2019benchmarking} (Line 1),  generates each random masking value from the same field $\mathbb{Z}_q$, and compute $b$, $\boldsymbol{U}$, $\boldsymbol{c}$, $\boldsymbol{g}$, and $a$, $h$ in $\mathbb{Z}_q$~\cite{catrina2010secure} (Lines 2-3);   
		
		\item $P_i$ only sends $h$ to $P_0$ but keeps $a$ private (Line 3); 
		
		\item \zyw{The} value $\beta - \alpha$ is jointly computed by all parties using MPC (Line 4).

	\end{enumerate}
	
	Note that the protocol incurs only one additional interactive operation when computing $\langle z \rangle = \langle \frac{1}{b} \rangle \langle a \rangle$. Thus, computing the Euclidean norm between $S$ and the $M-M_0$ series requires still $O(L_S) + O((N - L_S + 1)(M-M_0))$, which is \textit{much smaller} compared to \zyw{directly using} the MPC operations in $\Pi_{FedSS-B}$.

	More importantly, we verify the security guarantee of $\Pi_{DP}$.
	
	\begin{theorem}
		$\Pi_{DP}$ is secure under the security definition defined in Definition~\ref{definition:security}.
	\end{theorem}
	\begin{proof}[Proof Sketch]
		Since the \zyw{used secret}-sharing-based MPC \zyw{is} secure, we focus on the messages beyond it.
		We describe two simulators $\mathcal{S}_0$ and $\mathcal{S}_i$ that simulate the messages of the adversaries for party $P_0$ and $P_i$, respectively.
		
		We first present $\mathcal{S}_0$. Similar to Eq.~\ref{eq:h_alpha}, when receiving the message $h$, the adversary knows  
		\begin{equation}
			h = (\boldsymbol{e}_2^T \cdot \boldsymbol{v} + g_{L+1}\alpha)\ mod\ q.
		\end{equation}
	Since the masking values $\boldsymbol{e}_2^T$, $g_{L+1}$ and $\alpha$ are independently and uniformly sampled from $\mathbb{Z}_q$, the distribution of $h$ is equal to $h^\prime = g_{L+1}\alpha\ mod\ q$. In the ideal interaction, $\mathcal{S}_0$ independently and randomly chooses $\alpha$ and $g_{L+1}$ from $\mathbb{Z}_q$ to compute and send $h^\prime$ to the adversary.  Indeed, the views of the environment in both ideal and real interactions are indistinguishable.
		
		Next, we discuss $\mathcal{S}_i$. By executing $\Pi_{DP}$ in the real interaction, the adversary of $P_i$ receives $\boldsymbol{U}$, $\boldsymbol{c}$, $\boldsymbol{g}$. Both $\boldsymbol{c}$ and $\boldsymbol{g}$ are derived from independent and randomly chosen values. Thus, $\mathcal{S}_i$ can follow the same procedure to compute them. Without loss of generality, we assume $r = 1$ and $d=2$. Then, $\boldsymbol{U} = \boldsymbol{Q} \cdot \boldsymbol{X}$ follows
		\begin{equation}
			\begin{split}
			\left(\begin{array}{cc}
			\boldsymbol{Q}_{1,1}x_1 + \boldsymbol{Q}_{1,2}\boldsymbol{x}_{2,1}\ \ & \boldsymbol{Q}_{2,1}x_1 + \boldsymbol{Q}_{2,2}\boldsymbol{x}_{2,1}\\	
			\ldots\ \  & \ldots\\
			\boldsymbol{Q}_{1,1}x_L + \boldsymbol{Q}_{1,2}\boldsymbol{x}_{2,L}\ \  & \boldsymbol{Q}_{2,1}x_L + \boldsymbol{Q}_{2,2}\boldsymbol{x}_{2,L} \\	
			\boldsymbol{Q}_{1,1} + \boldsymbol{Q}_{1,2}\boldsymbol{x}_{2,L+1}\ \ &
			\boldsymbol{Q}_{2,1} + \boldsymbol{Q}_{2,2}\boldsymbol{x}_{2,L+1}
			\end{array}\right)^T
			\end{split}.
		\end{equation}
		
		Note that we omit the modular operations at each entry for ease of exposition. The value of each entry is masked by a unique triplet, e.g., $(\boldsymbol{Q}_{11}, \boldsymbol{Q}_{12}, \boldsymbol{x}_{21})$ at the entry (1,1).  Because the values of these triplets are independently and randomly chosen from $\mathbb{Z}_q$, the elements of  $\boldsymbol{U}$ are independent and identically distributed. Similar to $\mathcal{S}_0$, $\mathcal{S}_i$ can simulate $\boldsymbol{U}$ by computing $\boldsymbol{U}^\prime$, where
		\begin{equation}
			\boldsymbol{U}^\prime_{i,j} = \boldsymbol{Q}_{i,k}\boldsymbol{x}_{i,j}\ mod\ q,\  k \in \{1,\ldots,d\},
		\end{equation}
		 and sends it along with $\boldsymbol{c}$, $\boldsymbol{g}$ to the adversary. The views of the environment in both ideal and real interaction are identically distributed.
		
		In summary, the simulators achieve the same effect \zyw{as} the adversaries achieve. The security follows.
	\end{proof}
	With the security guarantee, we can integrate $\Pi_{DP}$ into $\Pi_{FedSS-B}$ to accelerate the distance computation. The protocol $\Pi_{DP}$ can also serve as a building block for other applications.

	\section{Quality Measurement Acceleration}\label{measurement_acceleration}
	
	Empirically, evaluating the shapelet quality using IG with the binary strategy (Sec.~\ref{TSC model}) is the {state-of-the-art method in terms of TSC accuracy}. However, computing IG in the FL setting suffers from a {severe efficiency issue}. The reasons are concluded as follows. 
	
	\begin{enumerate}
		\item A large number ($M$) of thresholds will be evaluated for each candidate; 
		
		\item Evaluating different thresholds incurs duplicated interactive operations; 
		
		\item Evaluating one threshold is already inefficient mainly because the required secure division and logarithm operations are expensive (as illustrated in Fig.~\ref{fig:mpc_throughputs}); 
		
		\item The IG pruning strategies lose their efficacy due to the security issue (Sec.~\ref{protocol_bottleneck}). 
	\end{enumerate}

	To consider both accuracy and efficiency, we propose to speed up the quality measurement in $\Pi_{FedSS-B}$ in two aspects: 
	
	\vspace{0.5ex}
	\noindent
	\textbf{O1: Accelerating IG computation.} To benefit from IG in terms of TSC accuracy, we propose a speed-up method to reduce as many interactive operations as possible in computing IG based on \textit{secure sorting} (Sec.~\ref{sorting_based_acceleration}), which tackles the problem in reason 1.

	\noindent
	\textbf{O2: Tapping alternative measures.}  As the problems of 1, 3 and 4 are the \textit{inherent deficiencies} of IG which \zyw{are} difficult to avoid, we propose a trade-off method tailored for the FL setting by tapping other measures that are much more secure-computation-efficient than IG, at the cost of acceptable loss of TSC accuracy  (Sec.~\ref{trade_off}).

	\subsection{Sorting-based IG Acceleration}\label{sorting_based_acceleration}
	
	The straightforward IG computation in Sec.~\ref{protocol_description} is inefficient since it incurs $O(M^2)$ \textit{secure comparisons} for $\langle \boldsymbol{\gamma}_L \rangle$, and $O(M^2)$ \textit{secure multiplications} for $\langle |D_{S,y(S)}^{\tau, L}| \rangle$ and $\langle |D^{\tau, R}_{S,y(S)}| \rangle$. Inspired by~\cite{mueen2011logical} and~\cite{abspoel2020secure}, we propose to {securely reduce the duplicated interactive operations by pre-sorting the secretly shared distances and labels} before computing each $Q_{IG}(S)$.
	
	Assuming $\langle D_S \rangle = \bigcup_{i=0}^{n-1}\{\langle d_{T_j^i, S} \rangle\}_{j=1}^{M_i}$ are arranged in an \textit{ordered sequence}, i.e.,
	\begin{equation}
		\langle D_S^\prime \rangle = \{\langle d_j \rangle\}_{j=1}^M,\label{eq:sorting-a1}
	\end{equation}
	where
	\begin{equation}
		d_{j_1} < d_{j_2}, \forall\ 1 \le j_1 < j_2 \le M.\label{eq:sorting-a2}
	\end{equation}
	 \zyw{In this condition}, for each threshold $\langle \tau \rangle = \langle d_{j} \rangle$, we can get $\boldsymbol{\gamma}_L^\prime$ without using secure comparison, as:
	 \begin{equation}
	 	\boldsymbol{\gamma}_L^\prime = \boldsymbol{\gamma}_{D_S^{\prime \tau, L} \subseteq D_S^\prime}
	 \end{equation}
	  where
	  \begin{equation}
	  	\boldsymbol{\gamma}_L^\prime[j^\prime] = \begin{cases}
	  	1, \quad & j^\prime < j\\
	  	0, \quad & otherwise
	  	\end{cases}
	  \end{equation}
	  Meanwhile, if $\langle \boldsymbol{\gamma}_{c} \rangle (c \in \{1,\ldots,C\})$ is permuted into  $\langle \boldsymbol{\gamma}_{c}^\prime \rangle$ such that for each entry $j^\prime$,  $\langle \boldsymbol{\gamma}_{c}^\prime \rangle$ and $\langle D_S^\prime \rangle$ indicates the same sample, i.e., 
	  \begin{equation}
	  	\langle \boldsymbol{\gamma}_{c}^\prime \rangle[j^\prime] = \langle \boldsymbol{\gamma}^{i}_{TD^{i}_{c} \subseteq TD^{i}}\rangle[j] \iff \langle d_{j^\prime} \rangle = \langle d_{T_j^{i}, S} \rangle, \label{eq:sorting-a3}
	  \end{equation}
	   we can compute the statistics in Eq.~\ref{eq:IG_stats} by replacing $\langle \boldsymbol{\gamma}_L \rangle$, $\langle \boldsymbol{\gamma}_R \rangle$, $\langle \boldsymbol{\gamma}_{c} \rangle$ with $\boldsymbol{\gamma}_L^\prime$, $\boldsymbol{\gamma}_R^\prime = \boldsymbol{1} - \boldsymbol{\gamma}_L^\prime$, $\langle \boldsymbol{\gamma}_{c}^\prime \rangle$, respectively. 
	   Note that the newly produced $\boldsymbol{\gamma}_L^\prime$ is in plaintext thanks to the order of $\langle D_S^\prime \rangle$. Thus, \zyw{only $O(C)$ secure multiplications are required} to compute the statistics in Eq.~\ref{eq:IG_stats} for each threshold, where $C$ is a small constant representing the number of classes.

	Based on the above observation, the key to the acceleration is to \textit{securely sort} the secretly shared \textit{distances} and the \textit{indicating vectors} of the class labels. Note that to satisfy the security in Definition~\ref{definition:security}, no any intermediate information can be disclosed during the federated sorting, including not only the secretly shared values, but also their order. 
	
	Although we can protect each of the values using MPC, the common sorting algorithms, e.g., Quicksort or Merge sort, rely on the \textit{order information} of the sort keys to reduce complexity. As a result, the order information will be disclosed during \zyw{the execution of the algorithm}, which violates the security in Definition~\ref{definition:security}. We take Quicksort as an example to illustrate the leakage of the order information, as shown in Fig.~\ref{fig:example_quicksort}.

	To address the security problem while achieving a complexity smaller than $O(M^2)$, we adopt the sorting network~\cite{batcher1968sorting} to securely sort the distances. Given an input size, the sorting network has a \textit{fixed order of comparison operations}, regardless of the order of the input sequence~\cite{batcher1968sorting}. Therefore, we can protect both the value and the order information by \textit{performing the comparison and \zyw{swapping} operations using the secure comparison and assignment protocols} (see Sec.~\ref{mpc}) respectively. Fig.~\ref{fig:sorting_network} is a running example of combining the sorting network and the MPC protocols.
	
	The distances $\langle {D_S} \rangle$ \zyw{are} taken as the sorting key to permute both $\langle D_S \rangle$ and $\langle \boldsymbol{\gamma}_{c} \rangle$ ($c \in \{1,\ldots,C\}$) consistently. The output corresponds to the assumption in Eq.~\ref{eq:sorting-a1}-\ref{eq:sorting-a2} and~\ref{eq:sorting-a3}. The sorting network takes $O(M\log^2M)$ interactive operations for the input of size $M$~\cite{batcher1968sorting}. Thus, the complexity of computing each $\langle Q_{IG}(S) \rangle$ becomes $O(M\log^2M)$, which is much smaller than the $O(M^2)$ in $\Pi_{FedSS-B}$.
	
	\begin{theorem}
		The sorting-based acceleration is secure under the security definition defined in Definition~\ref{definition:security}.
	\end{theorem}
	\begin{proof}[Proof Sketch]
		The main difference between the acceleration method and the basic protocol for the IG computation is the usage of the sorting network, which is proved to be data-oblivious~\cite{bogdanov2014practical}. Thus, the security of the sorting-based acceleration follows.
	\end{proof}

	\begin{figure}[t]
		\centering
		\subcaptionbox{Example of Quicksort. The elements should be partitioned based on whether they are smaller than the pivot or not. Even performed in secret shares, the process will disclose the order information about the input.  \label{fig:example_quicksort}}{
			\includegraphics[width=.38\linewidth]{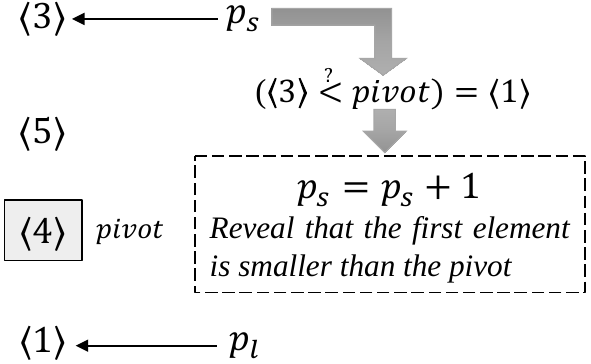}}\hspace{1ex}
		\subcaptionbox{The sorting network of size 4 used to sort the sequence. It contains 5 comparison operations which are determined only by the network size. Thus, by extending the sorting network using the secure comparison and assignment protocols, both the value and the order information can be protected.   \label{fig:sorting_network}}{
		\includegraphics[width=.54\linewidth]{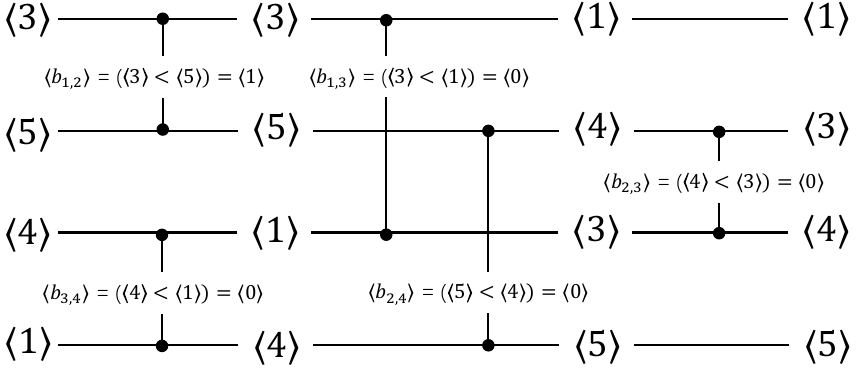}
		}\hspace{2ex}
		\caption{Running examples of sorting the sequence ($\langle 3 \rangle$, $\langle 5 \rangle$, $\langle 4 \rangle$, $\langle 1 \rangle$) using Quicksort and the sorting network.}
	\end{figure}
 
	\subsection{Alternative-Measures-based Trade-off}\label{trade_off}
	
	As discussed at the beginning
	of Sec.~\ref{measurement_acceleration}, {although IG-based method is superior in TSC accuracy, it is naturally difficult to efficiently compute this metric.} To further accelerate the quality measure step, we propose to tap \textit{alternative measures} that can be \zyw{achieved securely and more efficiently} in the FL setting, while guaranteeing comparable TSC accuracy.

	The shapelet quality can be evaluated \zyw{using} other measures, such as Kruskal-Wallis (KW)~\cite{lines2012alternative}, Mood's Median (MM)~\cite{lines2012alternative}, and ANOVA F (F-stat) test~\cite{hills2014classification}. However, these quality measures are less considered in recent works~\cite{bostrom2017binary,bagnall2020tale,middlehurst2021hive}, since they have no significant advantage over IG in terms of both accuracy and efficiency, especially when the binary strategy~\cite{bostrom2017binary} and the IG pruning technique~\cite{mueen2011logical} are integrated. {In the brand new FL scenario, the \textit{expensive communication cost} incurred by interactive operations and the \textit{failure of the pruning} for computing IG remind us to reexamine these alternatives.}

	{\setlength{\parindent}{0cm}
		\textbf{F-stat-based quality measurement.} As shown in~\cite{hills2014classification}, using F-stat for TSC slightly outperforms the methods with KW and MM in terms of accuracy. More essentially, F-stat performs with $O(M)$ secure multiplications and $C+1$ secure divisions in the FL setting, while both KW and MM require $O(M\log^2M)$ secure comparison and assignment operations because they rely on secure sorting, and they also need $C$ times of divisions. Thus, we choose F-stat as \zyw{an} alternative measure to achieve the trade-off.
	}

	Given $D_S = \{d_{T_j, S}\}_{j=1}^M$ and  $\{y_{j}\}_{j=1}^M$ where $y_j \in \{c\}_{c=1}^C$, the F-stat is defined as:
	\begin{equation}
	Q_F(S) = \frac{\sum_{c=1}^C (\overline{D}_{S,c} - \overline{D}_S)^2 / (C - 1)}{\sum_{c=1}^C \sum_{y_j = c} (d_{T_j, S} - \overline{D}_{S, c})^2 / (M - C)},
	\label{eq:F}
	\end{equation}
	where $\overline{D}_{S,c} = \frac{\sum_{d \in D_{S,c}} d}{|D_{S, c}|}$ is the mean distance w.r.t. class $c$ with $D_{S, c} = \{d_{T_j, S}|y_j=c\}_{j=1}^M$, and $\overline{D}_S$ is the mean of all distances.

	Similar to the \zyw{secure computation of IG} in Sec.~\ref{protocol_description},  we leverage the \textit{indicating vector} to indicate whether each sample belongs to each of the $C$ classes. Given 
	\begin{equation}
		\langle D_S \rangle = \bigcup_{i=0}^{n-1} \{\langle d_{T_j^i,S} \rangle\}_{j=1}^{M_i},
	\end{equation} 
	and the indicating vector $\langle \boldsymbol{\gamma}_{c} \rangle$ ($ c \in \{1,\ldots,C\}$) as: 
	\begin{equation}
		\langle \boldsymbol{\gamma}_{c} \rangle  =  (\langle \boldsymbol{\gamma}^0_{TD^0_{c} \subseteq TD^0}\rangle, \ldots,\langle \boldsymbol{\gamma}^{n-1}_{TD^{n-1}_{c} \subseteq TD^{n-1}}\rangle), 
	\end{equation}
the parties jointly compute the terms:
\begin{equation}
	\begin{split}
	\langle \overline{D}_{S, c} \rangle &= \frac{\langle \boldsymbol{D_S} \rangle \cdot \langle \boldsymbol{\gamma}_{c} \rangle}{\langle \boldsymbol{\gamma}_{c} \rangle \cdot \boldsymbol{1} },  c \in \{1,\ldots,C\}\\
	\langle \overline{D}_S \rangle &= \frac{\langle \boldsymbol{D_S} \rangle \cdot \boldsymbol{1}}{M}.
	\end{split}
\end{equation}

Next, they jointly compute:
\begin{equation}
	\sum_{y_j = c} (d_{T_j, S} - \overline{D}_{S, c})^2 =  \langle \boldsymbol{d_{c}} \rangle \cdot \langle \boldsymbol{d_{c}} \rangle, c \in \{1,\ldots,C\},
\end{equation}
 where 
 \begin{equation}
 	\langle \boldsymbol{d_{c}} \rangle[j] = \langle \boldsymbol{\gamma}_{c} \rangle[j] \cdot (\langle d_{T_j, S} \rangle - \overline{D}_{S, c}), j \in \{1,\ldots,M\}.
 \end{equation}

Then, the parties can jointly compute $\langle Q_F(S) \rangle$ by Eq.~\ref{eq:F}.
	
	The protocol for $\langle Q_F(S) \rangle$ has a complexity of $O(M)$, while the computation of $\langle Q_{IG}(S) \rangle$ using our optimization in Sec.~\ref{sorting_based_acceleration} still takes $O(M\log^2M)$ interactive operations. Moreover, the empirical evaluation in Sec.~\ref{exp:efficiency} shows that the F-stat-based FedST achieves \zyw{the accuracy comparable to that of} the prior IG-based solution.
	\begin{theorem}
		The F-stat-based shapelet quality measurement is secure under the security definition defined in Definition~\ref{definition:security}.
	\end{theorem}
	\begin{proof}[Proof Sketch]
		Similar to the IG-based method in Sec.~\ref{protocol_description} and Sec.~\ref{sorting_based_acceleration}, the input and output of the F-stat are both secret shares. The MPC operations and indicating vectors are used to make the computation data-oblivious. The security follows.
	\end{proof}
	
	\section{Experiments}\label{exp}
	
	In this section, we empirically evaluate the effectiveness of the \texttt{FedST} method and the acceleration techniques. 
	
	\subsection{Experimental Setup}\label{exp:setup}
	
	Our experimental setups are as follows:
	
	{
		\setlength{\parindent}{0cm}
		\textbf{Implementation.} \zy{FedST is implemented in Python}\zy{\footnote{\zy{https://github.com/hit-mdc/FedTSC-FedST.}}}. We \zyw{use} the SPDZ library~\cite{keller2020mp} for semi-honest additive-secret-sharing-based MPC. The security parameter is $\kappa = 40$, which ensures that the probability of information leakage, i.e., the quantity in Definition~\ref{definition:security} is less than $2^{-\lambda}$ ($\lambda=\kappa$)~\cite{catrina2010secure,catrina2010improved}.
		
	}

	{
		\setlength{\parindent}{0cm}
		\textbf{Environment.} We build a cross-silo federated learning environment by running parties in isolated 16G RAM and 8 core Platinum 8260 CPUs docker containers installed with Ubuntu 20.04 LTS. The parties communicate with each other through the docker bridge network with 4Gbps bandwidth.
	}
	
	{
		\setlength{\parindent}{0cm}
		\textbf{Datasets.} We use both the \textit{real-world datasets} and the \textit{synthetic datasets} for evaluation at the following two different scales.
	}
	
	To evaluate the effectiveness of \texttt{FedST} framework, we use the popular 117 fixed-length TSC datasets of the UCR Archive \cite{DBLP:journals/corr/abs-1810-07758} that are collected from different types of applications, such as ECG or motion recognition. In the cross-silo and horizontal setting, each business has considerable (but still insufficient) training samples for every class. Thus, we randomly partition the training samples into 3 equal-size subsets to ensure each party has at least two samples for each class. Since there are 20 small datasets that cannot be partitioned as above, we omit them and test on the remaining 97 datasets.
	
	To investigate the effectiveness of the acceleration techniques, we first assess the efficiency improvement of these techniques using the synthetic datasets. Since the secure computation is data-independent, we randomly generate the synthetic datasets of varying parameters. Next, we compare the F-stat to the prior IG measure in terms of both accuracy and  efficiency on the 97 UCR datasets to validate the effectiveness of the trade-off.
	
	{
		\setlength{\parindent}{0cm}
		\textbf{Metrics.} We use the \textit{accuracy} to evaluate the classification \zyw{performance}, which is measured as the number of samples that are correctly predicted over the testing dataset. For efficiency, we measure the \textit{running time} of the protocols in each step.
	}


	%
	%
	%

	\subsection{Effectiveness of the FedST Framework}\label{exp:accuracy}
	
	\textbf{Baselines.} Since the advantage of the shapelet transformation against other TSC methods has been widely shown~\cite{bagnall16bakeoff,bagnall2020tale,middlehurst2021hive}, we focus on investigating the \textit{effectiveness of enabling FL for TSC} in terms of classification accuracy. To achieve this goal, we compare our \texttt{FedST} with the four baselines:
	
	- \texttt{LocalST}: the currently available solution that $P_0$ performs the centralized shapelet transformation with only its own data;
	
	- \texttt{GlobalST}: the ideal solution that $P_0$ uses the data of all parties for centralized shapelet transformation without privacy protection;
	
	- \texttt{LocalS+FedT}: a variant of \texttt{FedST} that $P_0$ executes the shapelet search step locally and collaborates with the participants for \zyw{the} federated data transformation and classifier training;
	
	- \texttt{FedS+LocalT}: a variant of \texttt{FedST} that $P_0$ locally performs data transformation and classifier training using the shapelets found through \zyw{the} federated shapelet search.
	
	Following the centralized setting~\cite{bagnall2020tale,hills2014classification}, we adopt random forest as the classifier over the transformed data for all methods. The candidates are sampled with \zyw{a} length ranging from $min(3, \frac{N}{4})$ to $N$. The candidate set size is $\frac{MN}{2}$. \zy{The number of shapelets $K$ in Algorithm~\ref{alg:fedss_bs} is set to $\min\{\frac{N}{2},200\}$, while we reduce its size to 5 before data transformation and classifier training to improve interpretability. Inspired by~\cite{hills2014classification}, we cluster the $K$ shapelets into 5 groups using an agglomerative hierarchical clustering algorithm with Eq.\ref{eq:shapelet_dis} as the distance metric and select the centroids as the final shapelets.} The prior IG is used for assessing the shapelet quality.

 	\begin{figure}[]
		\centering
		\subcaptionbox{\centering \texttt{FedST} vs \texttt{LocalST} \label{subfig:vs local}}
		{\includegraphics[width=.49\linewidth]{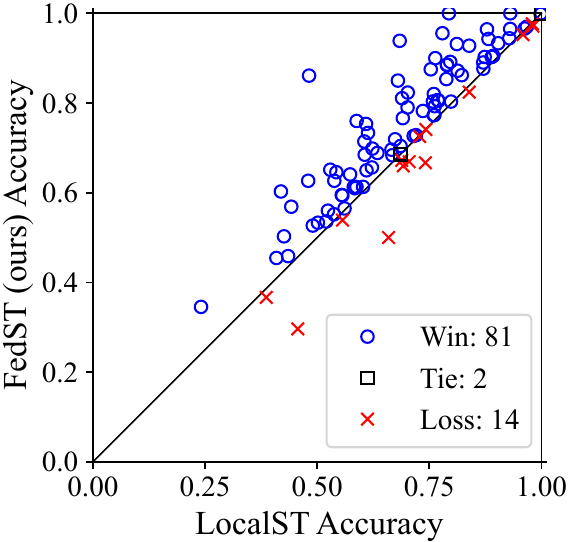}}
		\subcaptionbox{\centering \texttt{FedST} vs \texttt{GlobalST} \label{subfig:vs global}}
		{\includegraphics[width=.49\linewidth]{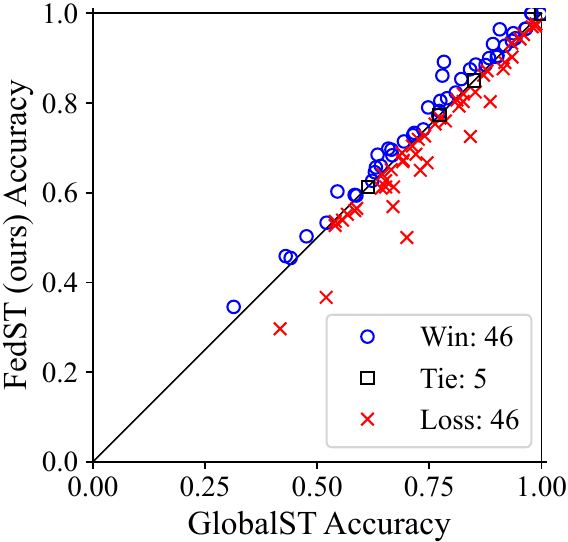}}
		\subcaptionbox{\centering \texttt{FedST} vs \texttt{LocalS+FedT} \label{subfig:vs locals}}
		{\includegraphics[width=.49\linewidth]{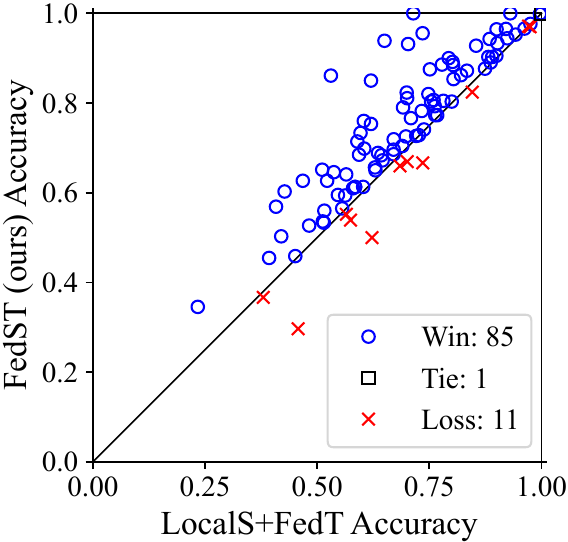}}
		\subcaptionbox{\centering \texttt{FedST} vs \texttt{FedS+LocalT} \label{subfig:vs localt}}
		{\includegraphics[width=.49\linewidth]{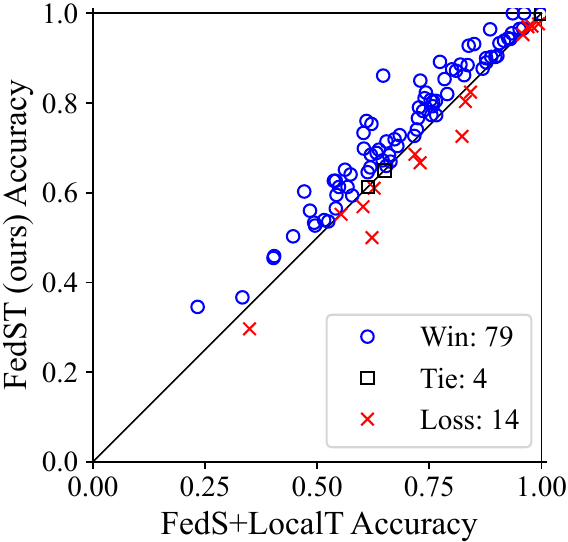}}
		\caption{Pairwise comparison between \texttt{FedST} and the baselines on 97 UCR datasets. The blue/black/red scatters represent the datasets where \texttt{FedST} wins/ties/loses the competitors.}
		\label{fig:1v1}
	\end{figure}
	
\begin{figure}[]
		\centering
		{\includegraphics[width=\linewidth]{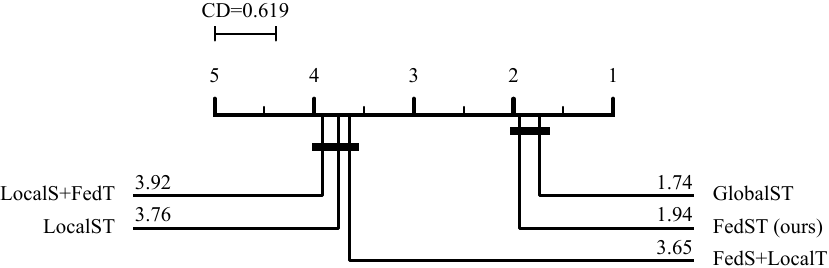}}
		\caption{Critical difference diagram for our \texttt{FedST} and the baselines under the statistical level of 0.05. }
		\label{fig:cd_ablation}
	\end{figure}

	{
		\setlength{\parindent}{0cm}
		\textbf{Pairwise comparison.} Fig.~\ref{fig:1v1} reports the pairwise accuracy comparison of \texttt{FedST} against the competitors.
	}
	
	Fig.~\ref{subfig:vs local} shows that \texttt{FedST} is more accurate than the \texttt{LocalST} on most of the datasets. It indicates the effectiveness of our basic idea of enabling FL to improve the TSC accuracy. Fig.~\ref{subfig:vs global} shows that \texttt{FedST} achieves accuracy close to the non-private \texttt{GlobalST}, which coincides with our analysis in Sec.~\ref{protocol_discussion}. The slight difference can be caused by two reasons. First, the global method samples the shapelets from all data, while in \texttt{FedST} the candidates are \zyw{generated only} by $P_0$ for the interpretability constraints. Second, in secret sharing, the float values are encoded in fixed-point representation for efficiency, which results in the truncation. Fortunately, we show later in Fig.~\ref{fig:cd_ablation} that there is \textit{no statistically significant difference} in accuracy between \texttt{FedST} and \texttt{GlobalST}. From Fig.~\ref{subfig:vs locals} and Fig.~\ref{subfig:vs localt}, we can see that the two variants are much worse than \texttt{FedST}. It means that both stages of \texttt{FedST} are indispensable.

	\begin{figure*}[t]
		\centering
		\subcaptionbox{\centering Time  w.r.t. $M$ \label{subfig:dis_M}}
		{\includegraphics[width=.29\linewidth]{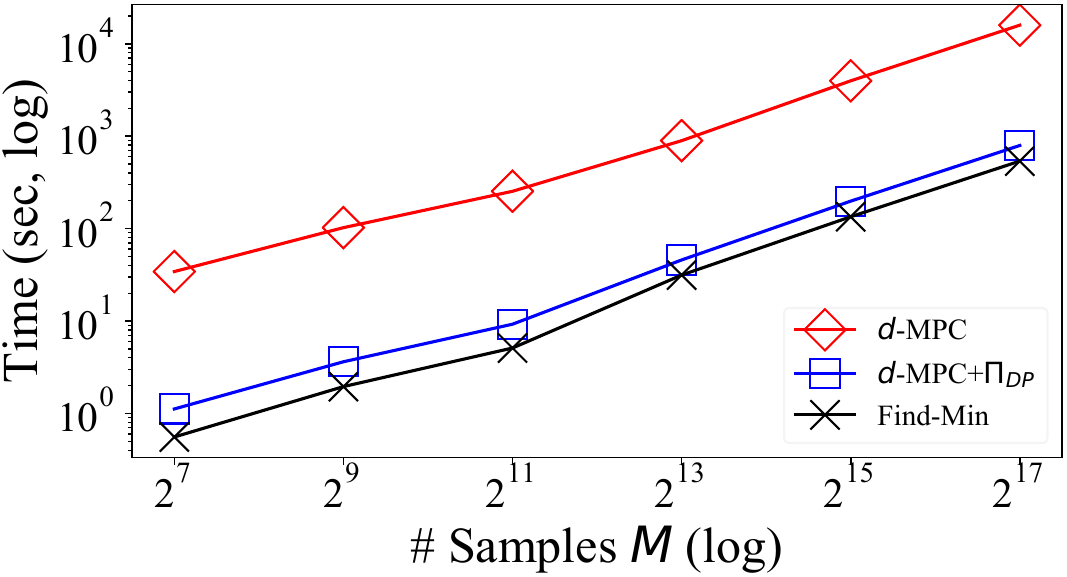}}
		\subcaptionbox{\centering Time  w.r.t. $N$ \label{subfig:dis_N}}
		{\includegraphics[width=.29\linewidth]{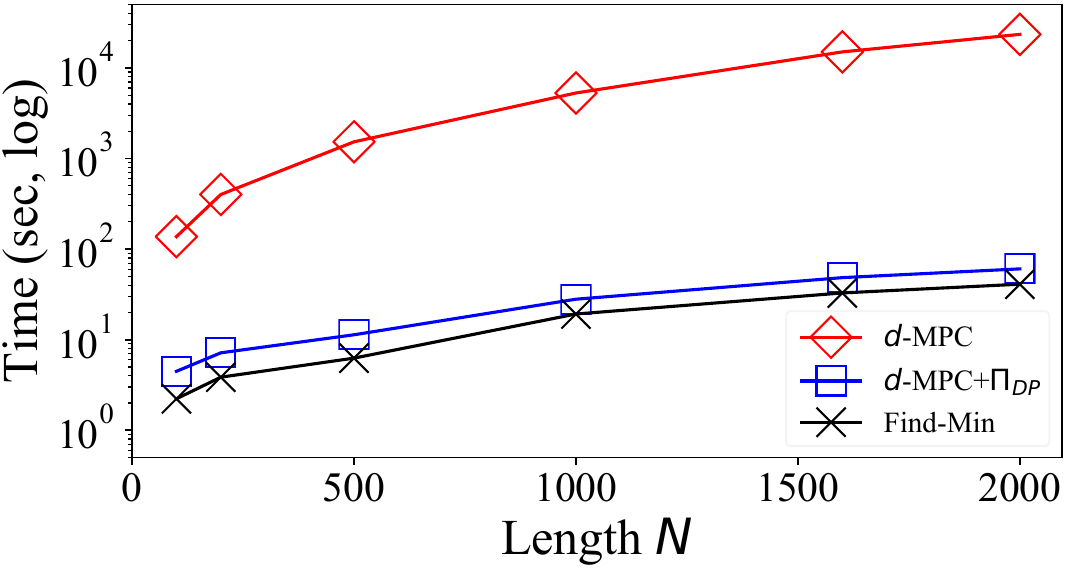}}
		\subcaptionbox{\centering Time  w.r.t. $n$ \label{subfig:dis_n}}
        {\includegraphics[width=.29\linewidth]{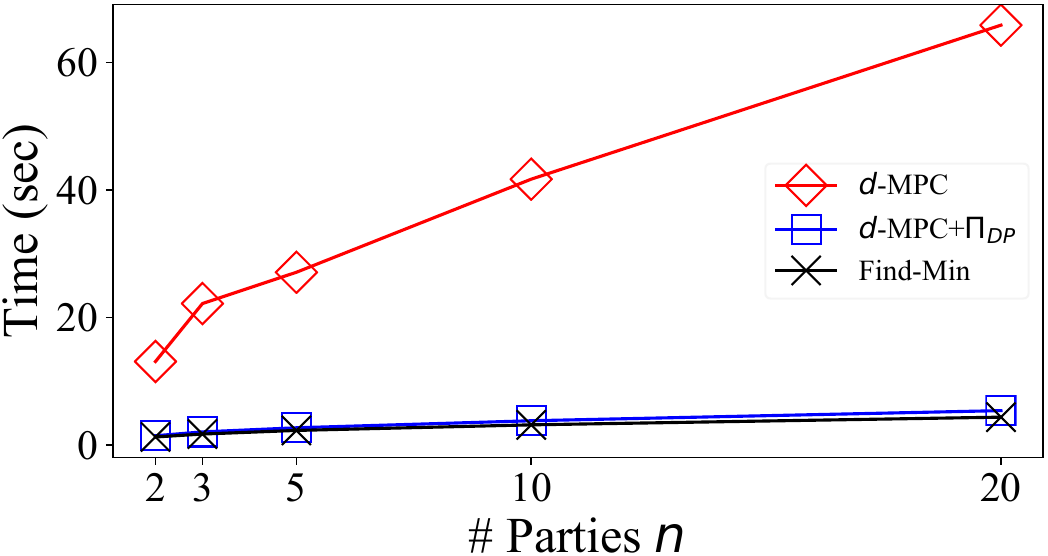}}
		\caption{\zy{Time of distance computation with respect to varying dataset size $M$ (default 512), series length $N$ (default 100), and the number of parties $n$ (default 3). }}
		\label{fig:efficiency_overall_distance}
	\end{figure*}

	\begin{figure}[t]
		\centering
		\subcaptionbox{\centering Time  w.r.t. $M$ \label{subfig:measure_M}}
		{\includegraphics[width=.49\linewidth]{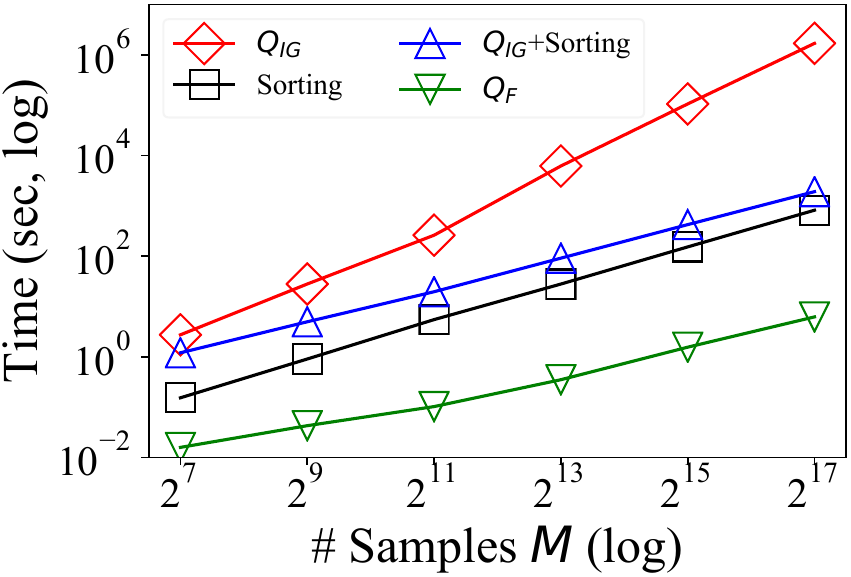}}
		\subcaptionbox{\centering Time  w.r.t. $n$ \label{subfig:measure_n}}
		{\includegraphics[width=.49\linewidth]{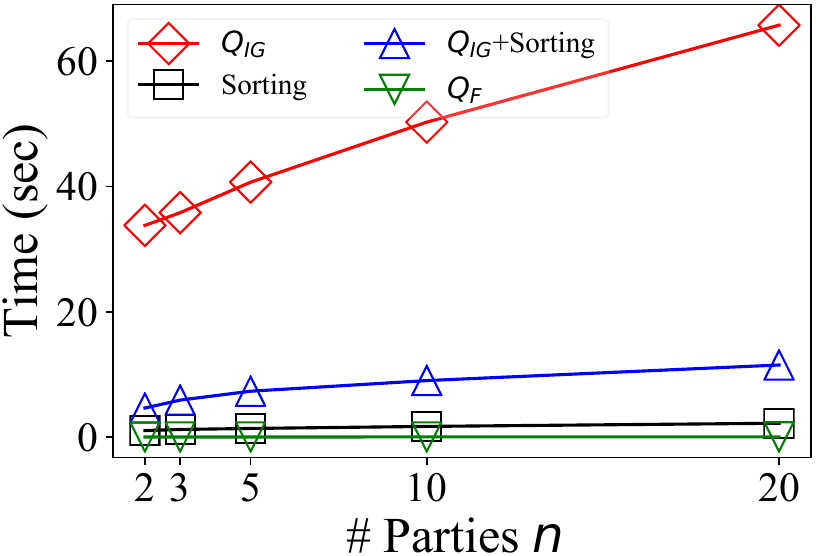}}
		\caption{\zy{Time of quality measurement with respect to varying dataset size $M$ (default 512) and the number of parties $n$ (default 3). }}
		\label{fig:efficiency_overall_quality}
	\end{figure}

	{
		\setlength{\parindent}{0cm}
		\textbf{Multiple comparisons.} We present the critical difference diagram~\cite{demvsar2006statistical} of the methods in Fig.~\ref{fig:cd_ablation}. It reports the \textit{mean ranking of accuracy} among the 97 UCR datasets. The competitors falling in one clique (the bold horizontal line) have no statistical difference, while the opposite for the methods from different cliques. Fig.~\ref{fig:cd_ablation} shows that \texttt{FedST} is \textit{no statistically different} from \texttt{GlobalST} and they both statistically significantly \zyw{outperform} \texttt{LocalST}. 
		It is notable that the variant conducting only local shapelet search (\texttt{LocalS+FedT}), even using all parties' data for transformation, is slightly inferior to \texttt{LocalST}. The reason could be that the locally selected shapelets \zyw{are of} poor quality due to the lack of training data, which may cause the transformed data to be more misleading to degrade the accuracy. In comparison, the variant \texttt{FedS+LocalT} performs better than \texttt{LocalST}, because the shapelet quality is improved by FL with more training data used for shapelet search. Both variants are much inferior to \texttt{FedST}, which indicates the positive effect of FL for both stages.
		%
	}

	\subsection{Effectiveness of the Acceleration Techniques}\label{exp:efficiency}
	
	\textbf{Efficiency improvement.} To assess the effectiveness of the proposed acceleration approaches, we first investigate their \textit{efficiency improvement} using the synthetic datasets of varying dataset size ($M$), time series length ($N$), \zyw{the} number of parties ($n$) and candidate set size $|SC|$.  The average length of the shapelet candidates and the number of shapelets ($K$) are fixed \zyw{at} moderate values of $0.6N$ and 200, respectively. Overall, the results in \zy{Fig.~\ref{fig:efficiency_overall_distance}-\ref{fig:efficiency_overall_search}} coincide with our complexity analysis.
	
	{
		\setlength{\parindent}{0.2cm}
		\textbf{1) Distance computation.} \zy{Fig.~\ref{subfig:dis_M}-\ref{subfig:dis_n}} report the time of computing the shapelet distance between a candidate $S$ and all training samples $T^i_j$ w.r.t. $M$, $N$, and $n$.  The time for both $\Pi_{FedSS-B}$ that directly uses MPC (\texttt{\texttt{$d$-MPC}}) and the optimization leveraging the proposed secure dot-product protocol (\texttt{\texttt{$d$-MPC}+$\Pi_{\mathtt{DP}}$}) scale linearly to $M$ and $n$. However, \texttt{\texttt{$d$-MPC}+$\Pi_{\mathtt{DP}}$} can achieve up to 30x of speedup over \texttt{$d$-MPC} for the default $N=100$. The time of \texttt{$d$-MPC} increases more quickly than \texttt{$d$-MPC}+$\Pi_{\mathtt{DP}}$ as $N$ increases, because the complexity of \texttt{$d$-MPC} is quadratic w.r.t. $N$ while our proposed \texttt{$d$-MPC}+$\Pi_{\mathtt{DP}}$ has a linear complexity of interactive operations.
	}
	
	We also show the time \zyw{to find} the minimum Euclidean norm (\texttt{Find-Min}), which is a subroutine of the shapelet distance computation. The results show that \texttt{Find-Min} is much faster than \texttt{$d$-MPC}, which is consistent with our analysis in Sec.~\ref{Distance_Acceleration} that the time of \texttt{$d$-MPC} is dominated by the \textit{Euclidean norm computation}. In comparison, the time of \texttt{$d$-MPC}+$\Pi_{\texttt{DP}}$ is very close to the time of \texttt{Find-Min} because \zyw{the time for the Euclidean norm computation} is substantially reduced (more than 58x speedup) with our $\Pi_{DP}$.
	
	{	\setlength{\parindent}{0.2cm}
		\textbf{2) Quality measurement.}} 
  We show the time of quality measurement for each candidate $S$ with varying $M$ and $n$ in \zy{Fig.~\ref{subfig:measure_M}-\ref{subfig:measure_n}}. Compared to the IG computation in the basic protocol ($\mathtt{Q_{IG}}$), our proposed secure-sorting-based method (\texttt{$\mathtt{Q_{IG}}$}+\texttt{Sorting}) achieves a similar performance when $M$ is small, but the time of $\mathtt{Q_{IG}}$ increases much faster than \texttt{$\mathtt{Q_{IG}}$+Sorting} as $M$ increases, because $\mathtt{Q_{IG}}$ has a quadratic complexity \zyw{with respect to} $M$. In comparison, the time of \texttt{$\mathtt{Q_{IG}}$+Sorting} is dominated by the secure sorting protocol (\texttt{Sorting}), which has a complexity of $O(M\log^2M)$. The optimized \texttt{$\mathtt{Q_{IG}}$+Sorting} is also more scalable to $n$ than $\mathtt{Q_{IG}}$.

	Using F-stat in the quality measurement step ($\mathtt{Q_{F}}$) can achieve more than 65x of speedup over the optimized \texttt{$\mathtt{Q_{IG}}$+Sorting}. It is also noteworthy that $\mathtt{Q_{F}}$ is much faster than \texttt{Sorting} which bottlenecks the time of securely computing the KW and MM, as mentioned in Sec.~\ref{trade_off}. That is why we consider the F-stat for the acceleration.

	{	\setlength{\parindent}{0.2cm}
		\textbf{3) Federated shapelet search.} Finally, we assess the \textit{total running time of the federated shapelet search protocol} with each proposed acceleration technique. The results are reported in \zy{Fig.~\ref{subfig:search_M}-\ref{subfig:search_SC}}.
	}
	
	Overall, an individual $\Pi_{DP}$-based acceleration ($+\Pi_{\mathtt{DP}}$) brings 1.01-73.59x of improvement over $\Pi_{FedSS-B}$.  The sorting-based (\texttt{+Sorting}) technique gives 1.01-96.17x of speedup alone and the F-stat-based method (+$\mathtt{Q_{F}}$) individually achieves 1.01-107.76x of speedup. The combination of these techniques is always more effective than each individual. $\Pi_{DP}$-based and Sorting-based methods together (+$\Pi_{\mathtt{DP}}$+\texttt{Sorting}) contribute 15.12-630.97x of improvement, while the combination of the $\Pi_{DP}$-based and F-stat-based techniques (+$\Pi_{\mathtt{DP}}$+$\mathtt{Q_{F}}$) boosts the protocol efficiency by 32.22-2141.64x.
	
	We notice \zyw{in} \zy{Fig.~\ref{subfig:search_M}} that the time of $\Pi_{FedSS-B}$ is dominated by the distance computation when $M$ is small. In this case, $+\Pi_{\texttt{DP}}$ is more effective. With the increase of $M$, the quality measurement step gradually dominates the efficiency. As a result, the \texttt{+Sorting} and +$\mathtt{Q_{F}}$ play a more important role in acceleration. Similarly, \zy{Fig.~\ref{subfig:search_N}} shows that the efficiency is dominated by the quality measurement when $N$ is small and gradually dominated by the distance computation with $N$ increases. The acceleration techniques for these two steps are always complementary with each other.
	
	It is also worth noting that the time of all competitors is nearly \textit{in direct proportion to} $|\mathcal{SC}|$, as shown in \zy{Fig.~\ref{subfig:search_SC}}. The result is consistent with our analysis in Sec.~\ref{protocol_discussion} that the time for securely \textit{finding the top-$K$ candidates} (Algorithm~\ref{alg:fedss_bs} Line 10), which has a complexity of $O(K\cdot |\mathcal{SC}|)$, is \textit{negligible} compared to the time of distance computation and quality measurement. That is why we mainly dedicate to accelerating these two steps.

	\begin{figure}[t]
		\centering
		\subcaptionbox{\centering Time  w.r.t. $M$ \label{subfig:search_M}}
		{\includegraphics[width=.49\linewidth]{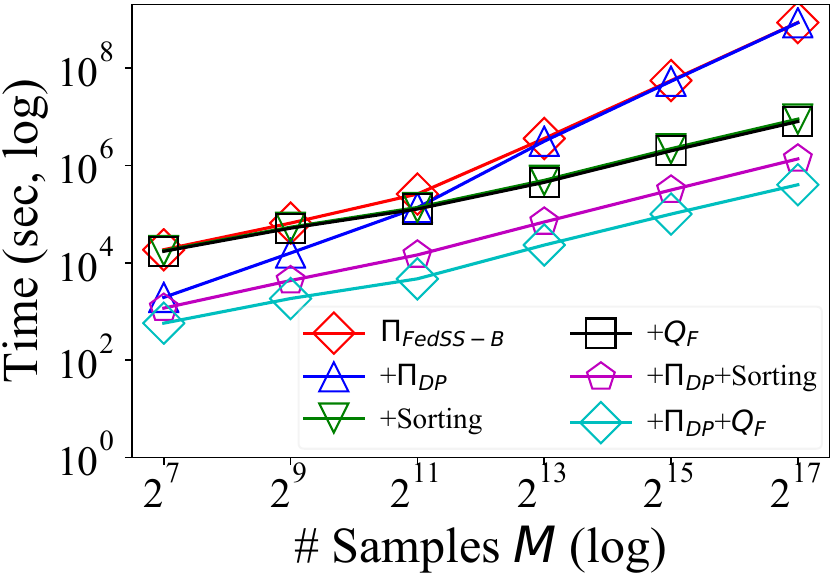}}
		\subcaptionbox{\centering Time  w.r.t. $N$ \label{subfig:search_N}}	
		{\includegraphics[width=.49\linewidth]{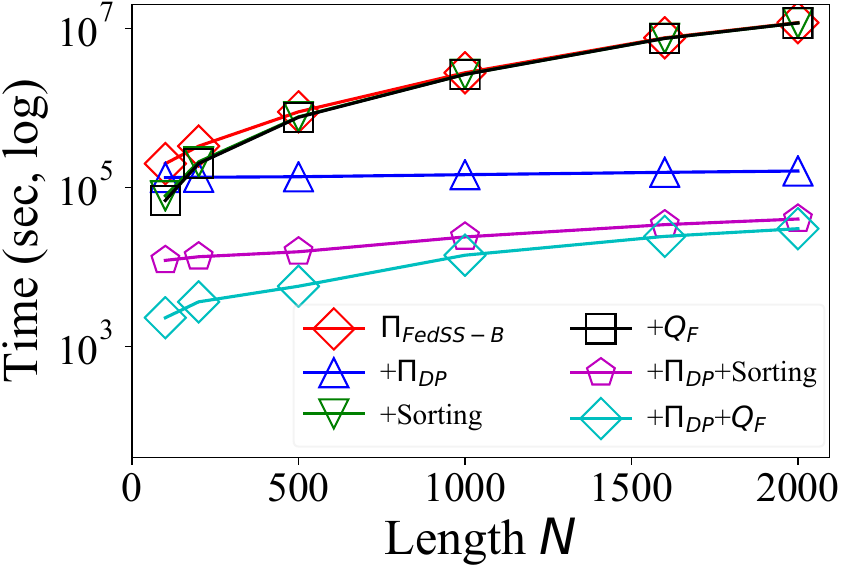}}
		\subcaptionbox{\centering Time  w.r.t. $n$ \label{subfig:search_n}}	
		{\includegraphics[width=.49\linewidth]{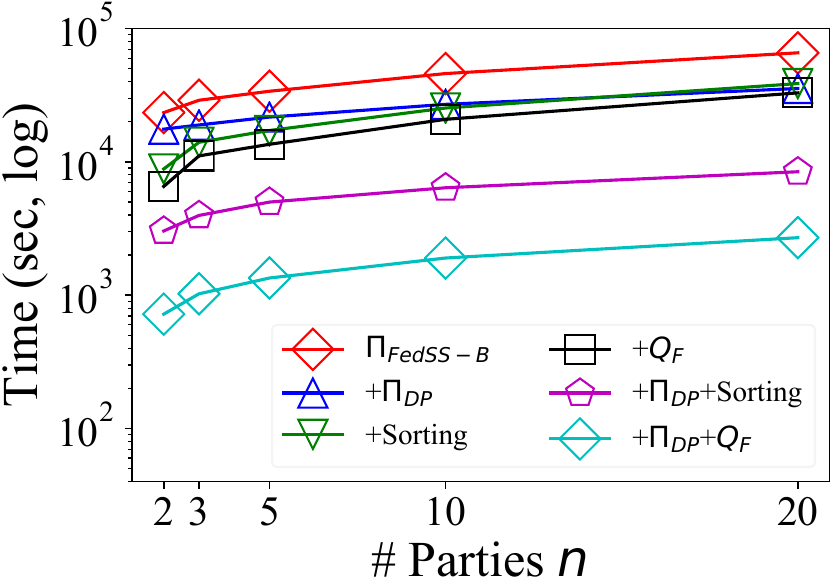}}
		\subcaptionbox{\centering Time  w.r.t. $|\mathcal{SC}|$ \label{subfig:search_SC}}	
		{\includegraphics[width=.49\linewidth]{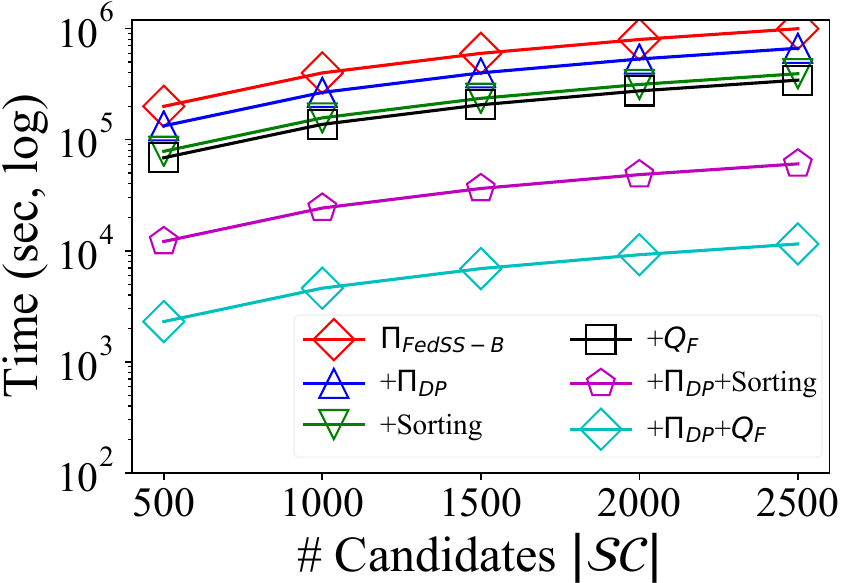}}
		\caption{\zy{Time of federated shapelet search with respect to varying dataset size $M$ (default 512), series length $N$ (default 100), the number of parties $n$ (default 3), and candidate set size $|\mathcal{SC}|$ (default 500). }}
		\label{fig:efficiency_overall_search}
	\end{figure}

\begin{table*}[]
    \centering
     \caption{\zy{Overall shapelet search time of \texttt{FedST} against the straightforward MPC-based solution \texttt{MPC}, and the non-private counterpart \texttt{NP}. The bold indicates the best in the corresponding category. }}
    \begin{tabular}{clrrrrr}
    \toprule
      \multicolumn{2}{c}{\textbf{Method}} & \textbf{Total (h)} & \textbf{Mean (min)} & \textbf{Median (min)} & \textbf{Max. (min)} & \textbf{Min. (min)} \\
    \midrule 
  \multirow{2}*{\makecell{Non-private}} &   NP-$Q_{IG}$ & 135.06 & 83.54 & 2.99 & \textbf{1186.72} & 0.00038\\
  ~ &      NP-$Q_{F}$ & \textbf{131.17} & \textbf{81.13} & \textbf{2.80} & 1191.19 & \textbf{0.00037}\\
     \hline\vspace{-2ex}& & & &\\
 \multirow{3}*{\makecell{\textbf{Private}}} &     MPC & $>9129.46$ & $>5647.09$ & 6428.97 & $>10080$ (1 week) & 0.74 \\
  ~ &   \textbf{FedST-$\boldsymbol{Q_{IG}}$} & 817.39 & 505.61& 148.75 & 5129.43 & 0.25\\
  ~ &     \textbf{FedST-$\boldsymbol{Q_{F}}$} & \textbf{625.91} & \textbf{387.16}& \textbf{60.45} & \textbf{4598.57} & \textbf{0.05}\\
    \bottomrule
    \end{tabular}
    \label{tab:search_time}
\end{table*}
 
	\begin{figure}
		\centering
		{\includegraphics[width=\linewidth]{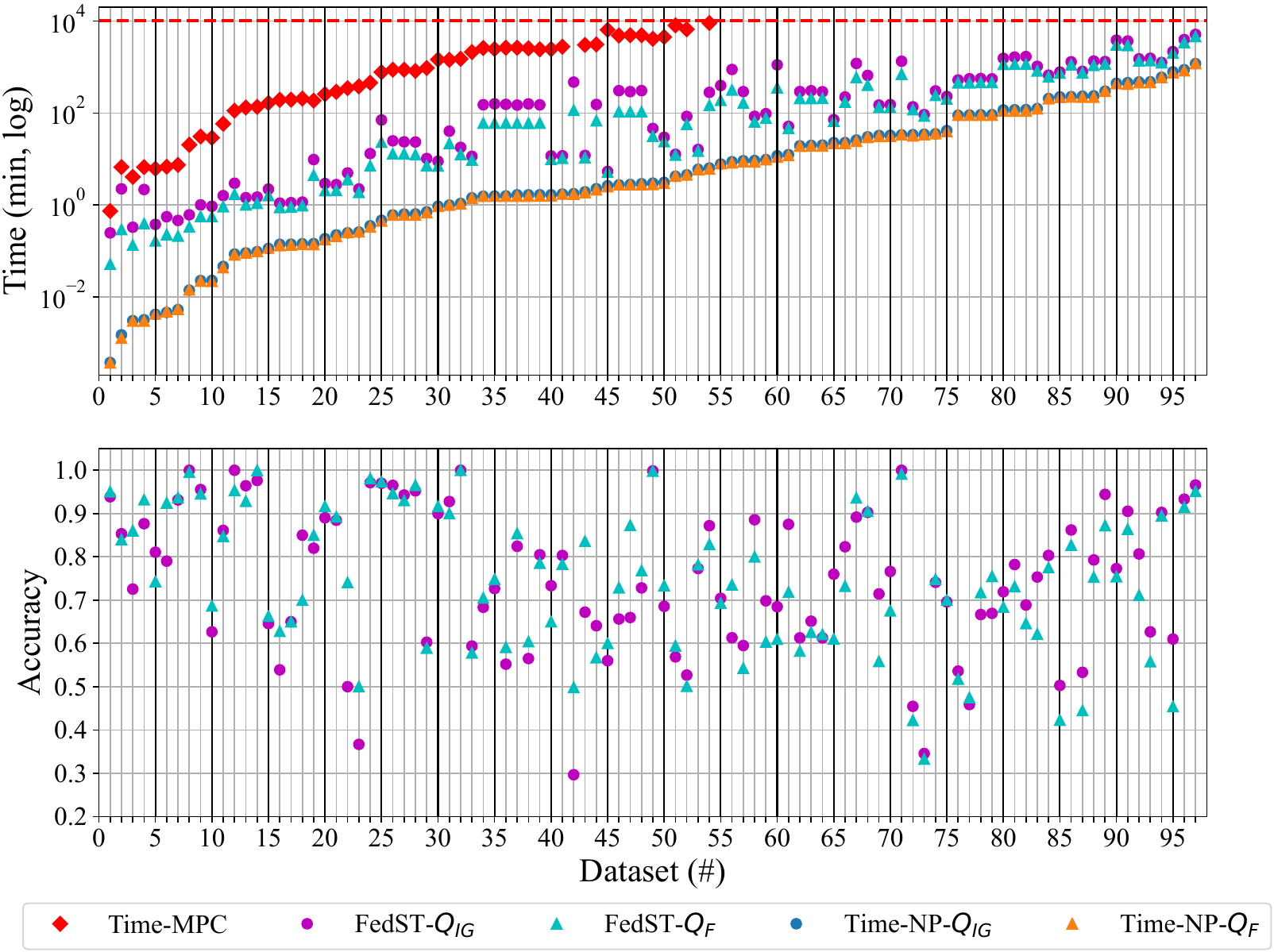}}
		\caption{\zy{Accuracy and federated shapelet search time of \texttt{FedST} using different quality measures. The horizontal red dashed line indicates the maximum running time we restrict, i.e. one week or 10080 minutes. The datasets are sorted according to \texttt{Time-NP-$\mathtt{Q_{IG}}$}. }}
		\label{fig:acc_vs_time}
	\end{figure}

\begin{figure}[]
		\centering
		{\includegraphics[width=\linewidth]{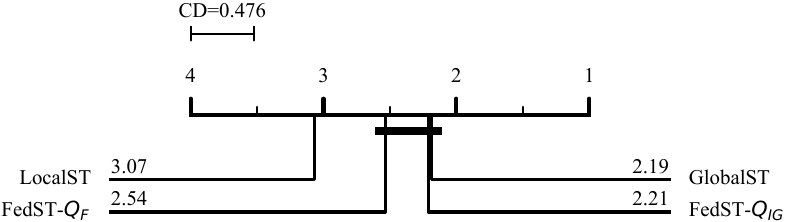}}
		\caption{Critical difference diagram for \texttt{FedST} that uses different quality measures and the two baselines. The statistical level is 0.05.}
		\label{fig:cd_main_accuracy}
	\end{figure}

	\noindent
		\textbf{Effectiveness of the trade-off strategy.} We investigate the effectiveness of the F-stat-based protocol in \textit{trading off TSC accuracy and the protocol efficiency.} Specifically, we evaluate both the accuracy and the federated shapelet search time for the two versions of \texttt{FedST} that adopt either the prior $Q_{IG}$ (\texttt{FedST-$\mathtt{Q_{IG}}$}) or the more efficient $Q_{F}$ (\texttt{FedST-$\mathtt{Q_{F}}$}). The experiments are conducted using 97 UCR datasets with the same setting as Sec.~\ref{exp:accuracy}. Both the $\Pi_{DP}$-based and the sorting-based speedup methods are adopted. \zy{We also provide the search time of the straightforward MPC-based protocol (\texttt{MPC}), and the non-private algorithms where $P_0$ directly uses the data of all parties (\texttt{NP-$\mathtt{Q_{IG}}$} and \texttt{NP-$\mathtt{Q_{F}}$}). They can be seen as the upper and lower bounds of the federated shapelet search time. Since \texttt{MPC} is quite slow, we restrict the maximum running time for a dataset to one week and ignore the cases running out of the time.}

	\zy{As shown in Fig.~\ref{fig:acc_vs_time} top, \texttt{MPC} runs out of time on 45 of the 97 datasets. Both our \texttt{FedST-$\mathtt{Q_{IG}}$} and \texttt{FedST-$\mathtt{Q_{F}}$} are much faster (at least an order of magnitude on average) than \texttt{MPC}. But they are still 6.05x and 4.77x slower than their non-private counterparts, as the cost of privacy protection. It may indicate that there is still room to improve the efficiency of the federated protocol.} \texttt{FedST-$\mathtt{Q_{F}}$} is faster than \texttt{FedST-$\mathtt{Q_{IG}}$} on all 97 datasets. The efficiency improvement is 1.04-8.31x while the average speedup on the 97 datasets is 1.79x.  
 
\texttt{FedST-$\mathtt{Q_{F}}$} is better than \texttt{FedST-$\mathtt{Q_{IG}}$} on 41 of the 97 datasets in terms of accuracy (Fig.~\ref{fig:acc_vs_time} bottom). The average accuracy of \texttt{FedST-$\mathtt{Q_{F}}$} is just $0.5\%$ lower than that of \texttt{FedST-$\mathtt{Q_{IG}}$}. Fig.~\ref{fig:cd_main_accuracy} shows the critical difference diagram for these two methods and the two FL baselines (\texttt{LocalST} and \texttt{GlobalST}). The result indicates that \texttt{FedST-$\mathtt{Q_{F}}$} achieves the same level of accuracy as \texttt{FedST-$\mathtt{Q_{IG}}$} and \texttt{GlobalST}, and is significantly better than \texttt{LocalST}. It indicates that our proposed F-stat-based strategy can effectively improve the efficiency of the federated shapelet search while guaranteeing comparable accuracy to the superior IG-based method.

\begin{figure}
		\centering
		\subcaptionbox{\centering Time  w.r.t. $M$ \label{subfig:hist_time}}
		{\includegraphics[width=.44\linewidth]{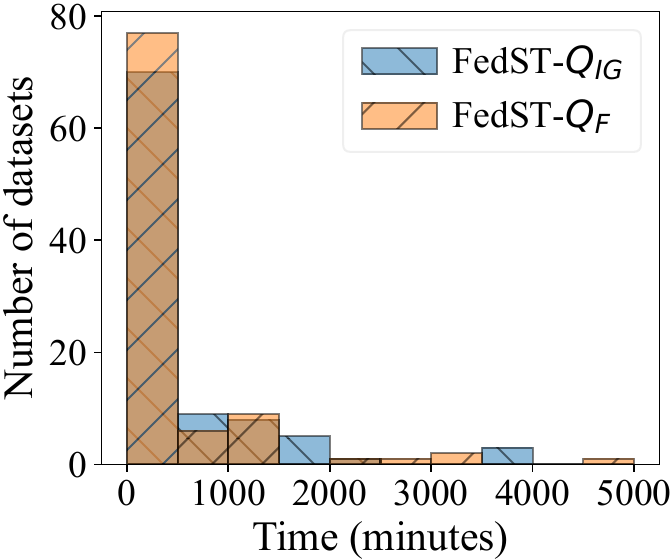}}
		\subcaptionbox{\centering Time  w.r.t. $n$ \label{subfig:hist_acc}}
		{\includegraphics[width=.43\linewidth]{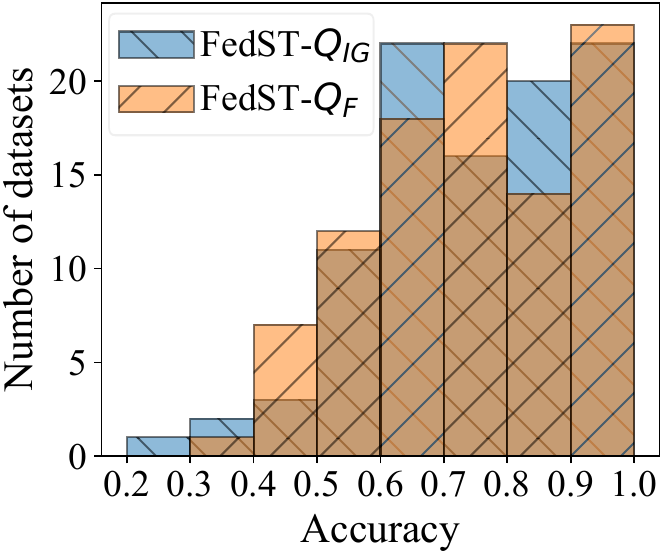}}
		\caption{\zy{Distributions of federated shapelet search time and accuracy of \texttt{FedST} using different quality measures.}}
		\label{fig:hist}
	\end{figure}	

 \zy{To further understand the results, we show in Fig.~\ref{fig:hist} the distributions of the federated shapelet search time and accuracy of our proposed methods. As shown in Fig.~\ref{subfig:hist_time}, both \texttt{FedST-$\mathtt{Q_{IG}}$} and \texttt{FedST-$\mathtt{Q_{F}}$} can finish in 500 minutes in most datasets. From Fig.~\ref{subfig:hist_acc}, we observe that the two variants achieve an accuracy greater than 0.6 in most cases. \texttt{FedST-$\mathtt{Q_{IG}}$} outperforms \texttt{FedST-$\mathtt{Q_{F}}$} in terms of the number of datasets in which the accuracy is greater than 0.8, but \texttt{FedST-$\mathtt{Q_{IG}}$} achieves the accuracy less than 0.4 in more datasets.}
	
\noindent
\textbf{\zy{Comparison of the overall search time.}}
 \zy{Finally, we investigate the overall time of shapelet search over the 97 datasets. As shown in  Table~\ref{tab:search_time},  our \texttt{FedST-$\mathtt{Q_{IG}}$} and \texttt{FedST-$\mathtt{Q_{F}}$} take 505.61 and 387.16 minutes on average per dataset, respectively, while the naive solution \texttt{MPC} requires more than 5647.09 minutes. The mean time of our federated solutions is comparable to the non-private counterparts, but is still several times longer, which may indicate the chance for further improvement. However, it should be noted that the privacy protection always comes at a cost, of either accuracy or efficiency.}

\begin{table*}[]
    \centering
     \caption{\zy{Overall training time of \texttt{FedST} against the standard TSC methods. The best is marked in bold, and the underlined value indicates the second best among the global and federated approaches. }}
    \begin{tabular}{clrrrrr}
    \toprule
       \multicolumn{2}{c}{\textbf{Method}} & \textbf{Total (h)} & \textbf{Mean (min)} & \textbf{Median (min)}& \textbf{Max. (min)} & \textbf{Min. (min)} \\
    \midrule 
    \multirow{3}*{\makecell{Local}} &    STC-Local & 520.60 & 322.02 & 320.28 & 381.40 & 320.02\\   
      ~&  DrCIF-Local & 48.31 & 29.88 & 11.54 &  290.43 & 0.56\\
      ~&  TDE-Local & 10.14 & 6.27 & 1.16 &  156.82 & 0.05\\
    \hline
    \vspace{-2ex}& & & &\\
     \multirow{3}*{\makecell{Global \&\\ Non-private}} &  STC-Global & 529.27 & 327.38  & 321.03 & \textbf{408.76} & 320.04\\
    ~&    DrCIF-Global & \underline{139.45} & \underline{86.26} & \underline{32.89} & \underline{891.62} & 1.33\\
    ~&    TDE-Global & \textbf{74.87} & \textbf{46.31} & \textbf{4.02} & 1420.88 & \textbf{0.06}\\
    \cline{1-1}
    \multirow{2}*{\makecell{\textbf{Federated} \\ \textbf{(Ours)}}} &   \textbf{FedST-$\boldsymbol{Q_{IG}}$} & 822.07 & 508.49& 149.79 & 5138.66 & 0.27\\
    ~&    \textbf{FedST-$\boldsymbol{Q_{F}}$} & 630.58 & 390.05& 61.54 & 4607.80 & \underline{0.07}\\
    \bottomrule
    \end{tabular}
    \label{tab:total_time}
\end{table*}

\subsection{\zy{Comparison with Standard TSC Approaches}}\label{exp:non-private}

\zy{To provide points of reference for our proposed method, we compare \texttt{FedST} with the state-of-the-art centralized TSC approaches, including the standard shapelet transformation method \texttt{STC}~\cite{bagnall2020tale}, the interval-based algorithm \texttt{DrCIF}~\cite{middlehurst2021hive}, and the dictionary-based approach \texttt{TDE}~\cite{tde}. These three algorithms are run by $P_0$ using either its local data or the global data of all parties, with the same hyper-parameter setting as used in HC2~\cite{middlehurst2021hive}.}

\zy{From Fig.~\ref{fig:cd_np}, we observe that the average accuracy ranking of both \texttt{FedST-$\mathtt{Q_{IG}}$} and \texttt{FedST-$\mathtt{Q_{F}}$} is higher than the \texttt{Local} competitors and lower than the \texttt{Global} ones.  \texttt{FedST-$\mathtt{Q_{IG}}$} shows no statistical differences against the standard methods \texttt{STC-Global} and \texttt{TDE-Global}, while it is statistically significant better than all \texttt{Local} variants of the standard TSC approaches. In comparison, the average accuracy ranking of  \texttt{FedST-$\mathtt{Q_{F}}$}, the version that trades efficiency with accuracy, is not significantly better (but still more accurate on average) than the \texttt{Local} competitors. The results further validate the effectiveness of our \texttt{FedST} framework. } 

\zy{Table~\ref{tab:total_time} shows the overall training time of the assessed methods. The \texttt{Global} version is always slower than the \texttt{Local} counterpart for all standard TSC approaches because more samples are used for training. Among the \texttt{Global} and \texttt{Federated} competitors that use the data of all parties, \texttt{TDE-Global} is the fastest on average, but its maximum time among the 97 datasets is longer than that of \texttt{STC-Global} and \texttt{DrCIF-Global}. \texttt{DrCIF-Global} and \texttt{STC-Global} are the second and third fastest on average, respectively. Our two variants of \texttt{FedST}, though little slower than the three \texttt{Global} competitors, ensure the protection of privacy required in the practical FL scenario.}

\begin{figure}[]
		\centering
		{\includegraphics[width=\linewidth]{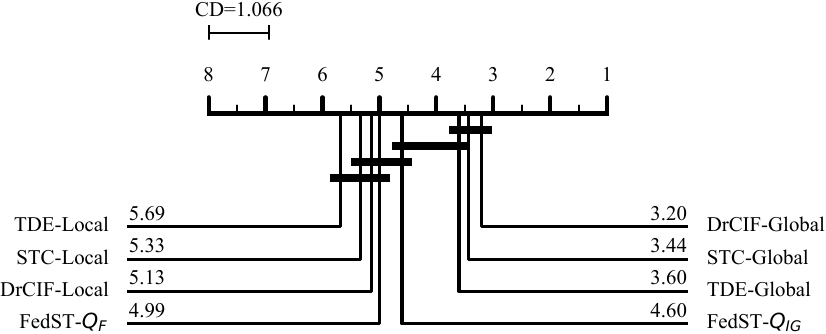}}
		\caption{\zy{Critical difference diagram for \texttt{FedST} that uses different quality measures and the standard non-private TSC methods. The statistical level is 0.05.}}
		\label{fig:cd_np}
	\end{figure}

\zy{To illustrate the training time distributions of the \texttt{Global} and \texttt{Federated} methods over the 97 datasets, we show the results in box plots in Fig.~\ref{fig:time-boxplot}. It is observed that the time of \texttt{STC} is very close for different datasets, because the computation time is limited to a fixed value following the standard setting used in HC2~\cite{middlehurst2021hive}. For most datasets, the time of our \texttt{FedST} variants is shorter than that of \texttt{STC} (because the setting of \texttt{FedST} leads to fewer shapelet candidates in these data sets) and comparable to the time consumed by \texttt{DrCIF}. \texttt{FedST} can be very efficient in some cases and rarely runs for more than one day using any quality measure.}

\zy{It is noteworthy that the interval-based \texttt{DrCIF} and dictionary-based \texttt{TDE} can be complementary with our shapelet-based framework to further improve the accuracy, as is widely validated in HC2~\cite{middlehurst2021hive}. They are also shown to be more efficient than the shapelet-based \texttt{STC} in the centralized scenario. Therefore, we will also consider them for developing the federated TSC solutions in our future work. }

\subsection{\zy{Comparison with FL Methods}}\label{exp:FL}

 \begin{figure}[]
		\centering
		{\includegraphics[width=.7\linewidth]{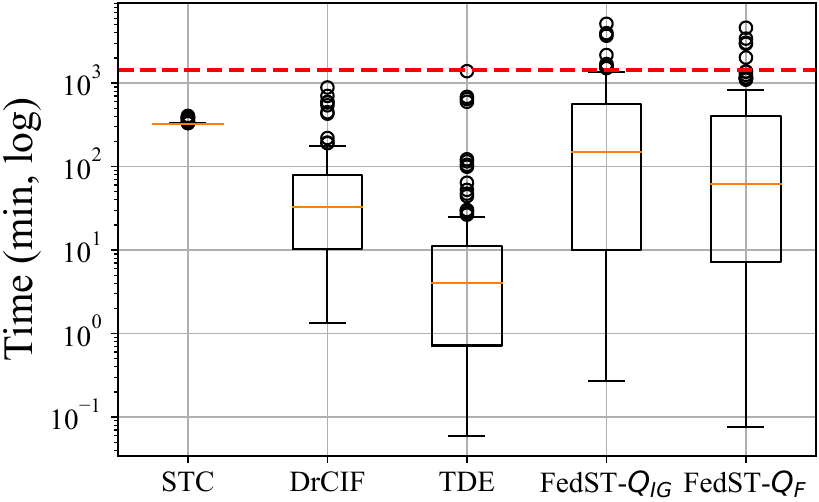}}
		\caption{\zy{Training time comparison between \texttt{FedST} and the standard non-private TSC methods. The red dashed line corresponds to one day, i.e. 1440 minutes.}}
		\label{fig:time-boxplot}
	\end{figure}

\begin{figure}[]
		\centering
		{\includegraphics[width=\linewidth]{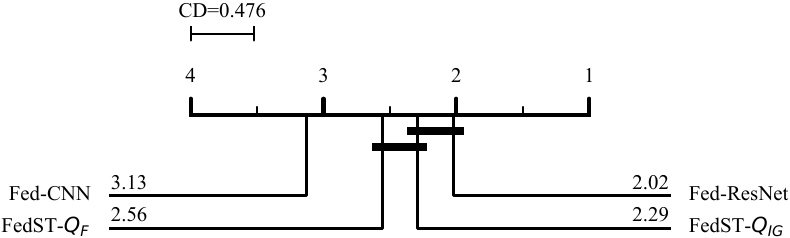}}
		\caption{\zy{Critical difference diagram for \texttt{FedST} that uses different quality measures and the FL methods. The statistical level is 0.05.}}
		\label{fig:cd_fl}
	\end{figure}

\zy{To further investigate the effectiveness of the proposed federated TSC solution, we compare it with two competitive FL baselines that use the popular \texttt{FedAvg} framework~\cite{mcmahan2017communication} to train two representative TSC models: a customized spatial-temporal Convolutional Neural Network~\cite{st-CNN} (denoted as \texttt{Fed-CNN}), and the state-of-the-art deep model ResNet~\cite{ResNet-TSC} (\texttt{Fed-ResNet}). We set the number of epochs for each round at 10 for \texttt{FedAvg} and the other hyper-parameters are the same as~\cite{ismail2019deep} for benchmarking. The accuracy result is shown in Fig.~\ref{fig:cd_fl}.}

\zy{Both \texttt{FedST} variants that use different quality measures significantly outperform \texttt{Fed-CNN}. \texttt{FedST-$\mathtt{Q_{IG}}$} is slightly inferior to \texttt{Fed-ResNet}, but there is no statistically significant difference. \texttt{FedST-$\mathtt{Q_{F}}$}, which trades efficiency with accuracy, achieves a moderate level of accuracy on average compared to the state-of-the-art \texttt{Fed-ResNet}. The result validates the competitive precision of our \texttt{FedST}.}

\zy{It is also noteworthy that our \texttt{FedST}  has two nice properties compared to the deep-learning-based FL approaches. First, \texttt{FedST} adopts several shapelet-based features that are intuitive-to-understand (see Sec.~\ref{exp:interpretability}) for training classifiers, which can be easier to interpret compared to deep neural networks that are generally seen as black boxes~\cite{interpretable-ML-book}. Second, the federated shapelet search algorithm can be run in an anytime fashion to flexibly balance accuracy and efficiency (see Sec.~\ref{exp:flexibility}), which is beneficial for practical utility. Moreover, the generic FL frameworks such as \texttt{FedAvg} and its variants rely on a secure broker that is costly and can disclose sensitive data~\cite{tong2022hu}, while our \texttt{FedST} solution does not need such a broker and is theoretically proven secure.}

	\subsection{Study of Interpretability}\label{exp:interpretability}
	Fig.~\ref{fig:interpretation} demonstrates the interpretability of \texttt{FedST} using a real-world motion classification problem named GunPoint~\cite{DBLP:journals/corr/abs-1810-07758}. The data track the centroid of the actors' right hand for two types of motions. For \zyw{the} ``Gun'' class, they draw a replicate gun from a hip-mounted holster, point it at a target, and then return the gun to the holster and their hands to their sides. For ``No gun (Point)'', the actors have their gun by their sides, point with their index fingers to a target, and then return their hands. The best shapelets of the two classes are shown in Fig.~\ref{subfig:GunPoint_shapelets}, which are derived from the data of the initiator and represent the class-specific features, i.e., the hand tracks of drawing the gun ($S_1$) and putting down the hand ($S_2$). We transform all time series into the distances to these shapelets and visualize the results in Fig.~\ref{subfig:GunPoint_ST}. As can be seen, instead of considering the original time series space which has 150 data points per sample, by using the two shapelets, the classification can be explained with the concise rule that the samples more similar to $S_1$ and distant from $S_2$ belong to class ``Gun'' (red), and the opposite for the ``No gun'' data (blue). \zy{The study indicates that the shapelet-based features are highly interpretable when the classes can be distinguished by some localized ``shapes'', which serves as a nice property of our \texttt{FedST}.}

	\subsection{Study of Flexibility}\label{exp:flexibility}
	We further investigate the flexibility of \texttt{FedST} as discussed in Sec.~\ref{fedst_kernel}. We evaluate the accuracy and the protocol running time on each of the 97 UCR datasets with the time contract varying from $10\%$ to $90\%$ of the maximum running time (the time evaluated in Sec.~\ref{exp:efficiency} using IG). Fig.~\ref{fig:flexibility} reports the results. Overall, the accuracy increases with more time allowed, while the real running time is always close to the contract. It validates the effectiveness of balancing the accuracy and the efficiency using the user-defined time contract, which is \zyw{beneficial} for practical utility. 
	
	Note that with only 10\% running time (approximate 10\% candidates assessed at random), FedST can achieve at least 77\% of the maximum accuracy among the 97 datasets, \zyw{implying that} the high-quality shapelets are highly redundant. The results also confirm the feasibility of generating candidates from $P_0$ in the cross-silo setting, where each party has considerable (but insufficient) data.

	\begin{figure}
	\centering
	\subcaptionbox{\centering The best shapelets of each class. \label{subfig:GunPoint_shapelets}}
	{\includegraphics[width=.52\linewidth]{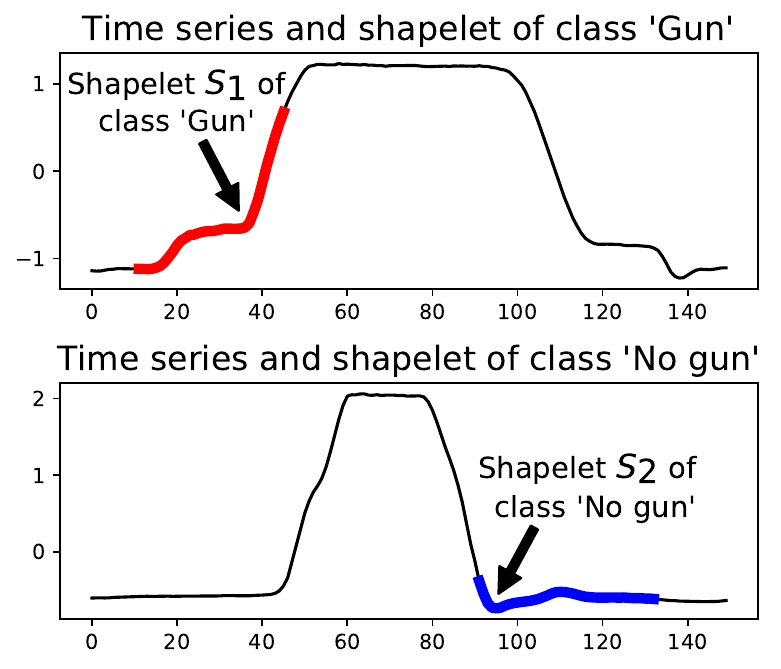}}
	\subcaptionbox{\centering The transformed data using the two shapelets. \label{subfig:GunPoint_ST}}
	{\includegraphics[width=.46\linewidth]{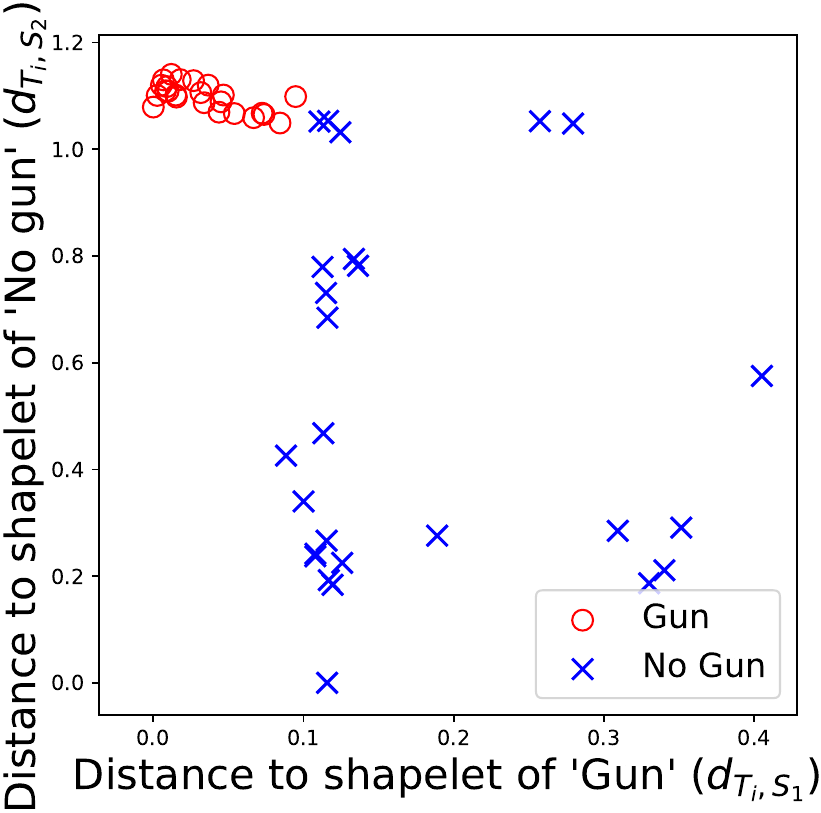}}
	\caption{Interpretability study using GunPoint~\cite{DBLP:journals/corr/abs-1810-07758}. }
	\label{fig:interpretation}
\end{figure}

\begin{figure}[]
		\centering
		{\includegraphics[width=.8\linewidth]{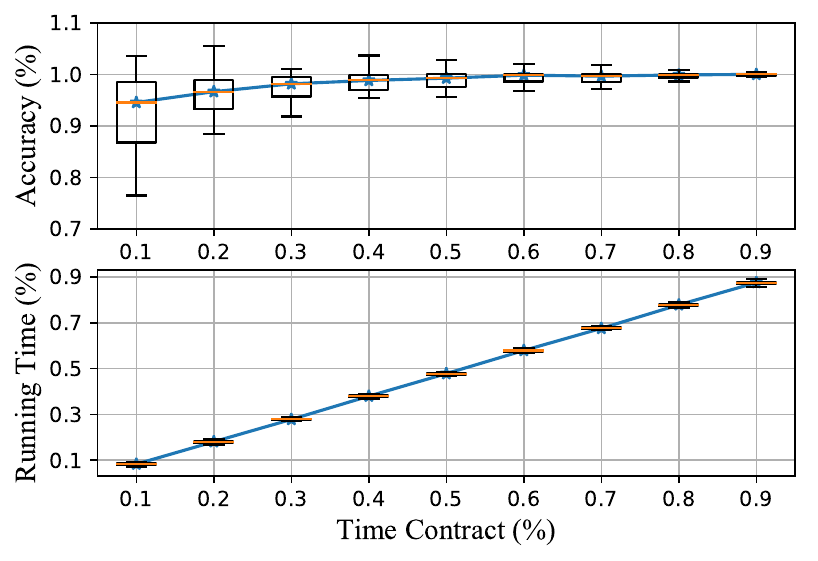}}
		\caption{The accuracy (top) and real running time (bottom) w.r.t. the user-defined time contract.}
		\label{fig:flexibility}
	\end{figure}

\subsection{\zy{Ablation Study of Shapelet Clustering.}}\label{exp:clustering}

\zy{Finally, we conduct an ablation study to assess the effectiveness of clustering the retrieved shapelets, which is the prior setting in \texttt{FedST} to simplify the interpretation. We compare it with two variants that use all retrieved shapelets (\texttt{FedST-Full}) or the top 5 shapelets of the highest quality (\texttt{FedST-Top}).} 

	 \begin{figure}[]
		\centering
		{\includegraphics[width=\linewidth]{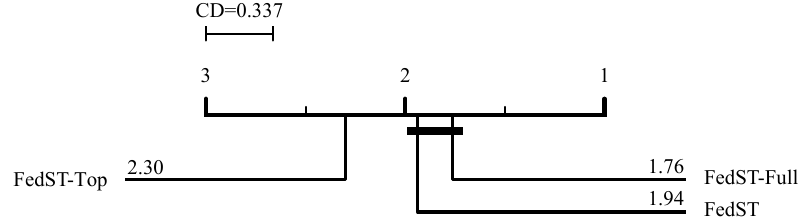}}
		\caption{\zy{Critical difference diagram for \texttt{FedST} against \texttt{FedST-Full} and \texttt{FedST-Top}. The statistical level is 0.05.}}
		\label{fig:cd_clustering}
	\end{figure}

\zy{As Fig.~\ref{fig:cd_clustering} shows, \texttt{FedST} with clustering achieves a mean accuracy ranking comparable to \texttt{FedST-Full}. \texttt{FedST-Top} is significantly worse than our \texttt{FedST}, because selecting too few shapelets based on quality scores can result in overfitting~\cite{hills2014classification}. That is why we choose to reduce the number of shapelets using clustering.}

\zy{Note that there is always a trade-off between accuracy and interpretability. Although it is effective to set the number of clusters to 5 in the assessed datasets, this hyper-parameter may be changed for other TSC problems to better balance the accuracy and the interpretability.}

\section{Further Enhancement by Incorporating Differential Privacy}\label{dp-protect}

As discussed in Sec.~\ref{fedst_framework}, FedST allows only the found shapelets and the learned models to be revealed to the initiator $P_0$. In this section, we illustrate that we can incorporate differential privacy (DP)~\cite{dwork2014algorithmic} for additional privacy protection, guaranteeing that the released shapelets and models disclose limited information about the private training data of the parties. The differential privacy is defined as follows.

\begin{definition}[Differential Privacy]
	Formally, a function $f$ satisfies $(\epsilon,\delta)$-DP, if for any two data sets $D$ and $D^\prime$ differing in a single record and any output $O$ of $f$, we have
	\begin{equation}
	\Pr[f(D)\in O] \le e^\epsilon \cdot \Pr[f(D^\prime)\in O] + \delta,
	\end{equation}
	where $\epsilon$ is the privacy budget controlling the tradeoff between the accuracy of $f$ and the degree of privacy protection that $f$ offers.
\end{definition}

 Intuitively, the function $f$ is differentially private since the probability of producing a given output (e.g., shapelets or models) is not highly dependent on whether a particular data record exists in $D$. As a result, the information about each private record cannot be inferred from the output with a high probability.

In FedST, we have two main stages: the federated shapelet search that produces the $K$ best shapelets, and the data transformation and classifier training step which builds the classification model and outputs the model parameters. Many existing work have studied DP algorithms to protect the parameters of the commonly used models~\cite{wu13privacy,abadi2016deep,chaudhuri2008privacy}, which can be seamlessly integrated into FedST. Therefore, we elaborate below on how to incorporate DP to the federated shapelet search.

As defined in Definition~\ref{definition:FedSS}, the federated shapelet search takes the parties' training time series and the shapelet candidates as input, and the output is the $K$ candidates with the highest quality ($Q_{IG}$ or $Q_F$ depending on the measure used). Note that the quality of each candidate is evaluated using all training time series. Therefore, we can \zyw{prevent} the private training samples from being disclosed by \textit{protecting the quality of each individual candidate}. To achieve this goal, we make the quality of the candidates \textit{noisy} before retrieving the $K$ best ones. 

Concisely, the parties jointly add secretly shared noises to the quality of all candidates using the secure random number generator~\cite{keller2020mp}. The noise for each candidate should be identically and independently distributed and follows a Laplace distribution whose parameter is public and related to $\epsilon$, which is referred to as the \textit{Laplace mechanism}~\cite{dwork2014algorithmic}. To this end, the parties retrieve the candidate with the \textit{maximum} quality by executing the secure comparison and assignment protocols (see Sec.~\ref{mpc}) and reveal the index to $P_0$. The two steps are repeated $K$ times to find the noisy $K$ best shapelets.

The above algorithm for finding the maximum is referred to as the \textit{Report Noisy Max} algorithm~\cite{dwork2014algorithmic}, which is $(\epsilon, 0)$-differentially private. Thus, according to Theorem 3 in~\cite{NEURIPS2019_b139e104}, the algorithm of retrieving the $K$ best shapelet candidates by calling the Report Noisy Max algorithm is $(\epsilon^\prime, \delta^\prime)$-DP for any $\delta^\prime \ge 0$ where
\begin{equation}
	\epsilon^\prime = \min \left\{ \epsilon K, \epsilon K(\frac{e^\epsilon-1}{e^\epsilon+1}) + \sqrt{2\epsilon^2 K\ln(\frac{1}{\delta^\prime})} \right\}.
\end{equation} 

In conclusion, the integration of DP provides an additional layer to protect the privacy of the revealed shapelets and models, which can further enhance the security of FedST.

	\section{\zyw{Conclusions and Future Work}}\label{sec:conclusion}
	
	\zyw{In this paper, we propose FedST, a novel FL framework customized for TSC based on the centralized shapelet transformation.} We design a security protocol $\Pi_{FedSS-B}$ for the FedST kernel, analyze its effectiveness, and identify its efficiency bottlenecks. To accelerate the protocol, we propose specific optimizations tailored for the FL setting. Both the theoretical analysis and the experimental results show the effectiveness of our proposed FedST framework and the acceleration techniques.
	
	In the future, we would like to consider other types of interpretable features to complement FedST. Further, we wish to develop \zyw{a} high-performance system to support industrial-scale applications.


\bibliographystyle{spmpsci}      
\bibliography{sample}   


\end{document}